\documentclass[11pt]{article}

\usepackage[utf8]{inputenc} 
\usepackage[T1]{fontenc}    
\usepackage[colorlinks=true, linkcolor=blue, citecolor=blue,urlcolor=black]{hyperref}       
\usepackage{url}            
\usepackage{booktabs}       
\usepackage{amsfonts}       
\usepackage{nicefrac}       
\usepackage{microtype}      
\usepackage{mkolar_definitions}
\usepackage[dvipsnames]{xcolor}
\usepackage{algorithm}
\usepackage[noend]{algorithmic}
\usepackage{wrapfig}

\usepackage{xspace}

\usepackage{enumitem}

\usepackage{breakcites}

\usepackage[suppress]{color-edits}
\definecolor{purple}{rgb}{0.6,0,1.0}
\addauthor{df}{red}
\addauthor{ak}{purple}
\addauthor{hl}{PineGreen}
\addauthor{wontfix}{orange}
\newcommand{\wontfix}[1]{\wontfixcomment{#1}}
\def\Reg{\ensuremath{\text{\rm Reg}}}

\usepackage[letterpaper, left=1in, right=1in, top=1in, bottom=1in]{geometry}

\usepackage{times}

\usepackage{parskip}

\usepackage[dvipsnames]{xcolor}
\usepackage{microtype}

\usepackage{natbib}
\bibpunct{(}{)}{;}{a}{,}{,}

\usepackage{amsthm}
\usepackage{mathtools}
\usepackage{amsmath}
\usepackage{bbm}
\usepackage{amsfonts}
\usepackage{amssymb}

\usepackage{MnSymbol} 

\usepackage{xpatch}

\xpatchcmd{\proof}{\itshape}{\normalfont\proofnameformat}{}{}
\newcommand{\proofnameformat}{\bfseries}

\DeclarePairedDelimiter{\abs}{\lvert}{\rvert} %
\DeclarePairedDelimiter{\brk}{[}{]}
\DeclarePairedDelimiter{\crl}{\{}{\}}
\DeclarePairedDelimiter{\prn}{(}{)}
\DeclarePairedDelimiter{\nrm}{\|}{\|}
\DeclarePairedDelimiter{\tri}{\langle}{\rangle}

\let\Pr\undefined

\DeclareMathOperator{\En}{\mathbb{E}}

\DeclareMathOperator{\Pr}{Pr}

\newcommand{\ls}{\ell}

\newcommand{\veps}{\varepsilon}

\newcommand{\ldef}{\vcentcolon=}

\newcommand{\wh}[1]{\widehat{#1}}

\def\ddefloop#1{\ifx\ddefloop#1\else\ddef{#1}\expandafter\ddefloop\fi}
\def\ddef#1{\expandafter\def\csname bb#1\endcsname{\ensuremath{\mathbb{#1}}}}
\ddefloop ABCDEFGHIJKLMNOPQRSTUVWXYZ\ddefloop
\def\ddefloop#1{\ifx\ddefloop#1\else\ddef{#1}\expandafter\ddefloop\fi}
\def\ddef#1{\expandafter\def\csname b#1\endcsname{\ensuremath{\mathbf{#1}}}}
\ddefloop ABCDEFGHIJKLMNOPQRSTUVWXYZ\ddefloop
\def\ddef#1{\expandafter\def\csname c#1\endcsname{\ensuremath{\mathcal{#1}}}}
\ddefloop ABCDEFGHIJKLMNOPQRSTUVWXYZ\ddefloop
\def\ddef#1{\expandafter\def\csname h#1\endcsname{\ensuremath{\widehat{#1}}}}
\ddefloop ABCDEFGHIJKLMNOPQRSTUVWXYZ\ddefloop
\def\ddef#1{\expandafter\def\csname hc#1\endcsname{\ensuremath{\widehat{\mathcal{#1}}}}}
\ddefloop ABCDEFGHIJKLMNOPQRSTUVWXYZ\ddefloop
\def\ddef#1{\expandafter\def\csname t#1\endcsname{\ensuremath{\widetilde{#1}}}}
\ddefloop ABCDEFGHIJKLMNOPQRSTUVWXYZ\ddefloop
\def\ddef#1{\expandafter\def\csname tc#1\endcsname{\ensuremath{\widetilde{\mathcal{#1}}}}}
\ddefloop ABCDEFGHIJKLMNOPQRSTUVWXYZ\ddefloop

\renewcommand{\paragraph}[1]{\par\textbf{#1}\hspace{5pt}}
\renewcommand{\defeq}{\ldef}

\newcommand{\pinv}{\dagger}

\newcommand{\subg}{\mathsf{subG}}
\newcommand{\sube}{\mathsf{subE}}
\newcommand{\subG}{\subg}
\newcommand{\subE}{\sube}

\usepackage[scaled=.90]{helvet}

\newcommand{\linucb}{\textsf{LinUCB}\xspace}
\newcommand{\modelselection}{\textsf{ModCB}\xspace}
\newcommand{\mainalgo}{\modelselection}
\newcommand{\mainalgolong}{Model Selection for Contextual Bandits}
\newcommand{\oracle}{\textsf{Oracle}\xspace}
\newcommand{\iltcb}{\textsf{ILOVETOCONBANDITS}\xspace}
\newcommand{\iltcbshort}{\textsf{ILTCB}\xspace}

\newcommand{\estimateresidual}{\textsf{EstimateResidual}\xspace}

\def\expix{\textsf{Exp4-IX}}

\newcommand{\algcomment}[1]{\textcolor{blue}{\footnotesize{\texttt{\textbf{//
          #1}}}}}
\newcommand{\algcommentbig}[1]{\textcolor{blue}{\footnotesize{\texttt{\textbf{/*
          #1~*/}}}}}

\newcommand{\cIbar}{\overline{\cI}}

\newcommand{\expexp}{\kappa}

\newcommand{\currentalg}{\textsf{\expix}}

\newcommand{\Picompact}{\Pi}

\newcommand{\trn}{\top}

\newcommand{\psdgt}{\succ}

\newcommand{\cov}{\Sigma}
\newcommand{\empcov}{\wh{\cov}}

\newcommand{\eigmin}{\lambda_{\mathrm{min}}}
\newcommand{\eigmax}{\lambda_{\mathrm{max}}}

\newcommand{\betas}{\beta^{\star}}
\newcommand{\pis}{\pi^{\star}}

\newcommand{\Otilde}{\tilde{O}}

\newcommand{\specnorm}{2}

\newcommand{\expone}{(1)}
\newcommand{\exptwo}{(2)}

\renewcommand{\Reg}{\mathrm{Reg}}

\newcommand{\ms}{m^{\star}}
\newcommand{\mh}{\wh{m}}

\newcommand{\Ls}{L^{\star}}

\newcommand{\pisq}{\pi^{\mathrm{sq}}}

\renewcommand{\mb}[1]{\boldsymbol{#1}}

\newcommand{\Tmin}{T^{\mathrm{min}}}

\usepackage{ifthen}
\newcommand{\version}{arxiv}

\title{{\huge Model selection for contextual bandits}}

\date{}
\author{%
  Dylan J. Foster\\
 {\small Massachusetts Institute of Technology}\\
  {\small dylanf@mit.edu}
  \and
  Akshay Krishnamurthy\\
  {\small Microsoft Research NYC}\\
  {\small akshay@cs.umass.edu}
  \and
  Haipeng Luo\\
  {\small University of Southern California}\\
  {\small haipengl@usc.edu}
}

\begin{document}

\maketitle

\begin{abstract}

We introduce the problem of model selection for contextual bandits, where a
learner must adapt to the complexity of the optimal policy while balancing exploration and exploitation. Our main result is a new model selection guarantee for linear contextual bandits. We work in the stochastic realizable setting with a sequence of nested linear policy classes of dimension $d_1 < d_2 < \ldots$,
where the $m^\star$-th class contains the optimal policy, and we design an
algorithm that achieves $\otil\prn{T^{2/3}d^{1/3}_{m^\star}}$
regret \emph{with no prior knowledge} of the optimal dimension
$d_{m^\star}$. The algorithm also achieves regret $\otil\rbr{T^{3/4} + \sqrt{Td_{m^\star}}}$,
which is optimal for $d_{\ms}\geq{}\sqrt{T}$. This is the first
contextual bandit model selection result with non-vacuous regret for
all values of $d_{m^\star}$, and to the best of our knowledge is the first positive result of this type for any online learning setting with partial information. The
core of the algorithm is a new estimator for the gap in the best loss
achievable by two linear policy classes, which we show admits a
convergence rate faster than the rate required to learn the parameters for either class.
\end{abstract}

\section{Introduction}

Model selection is the fundamental statistical task of choosing a
hypothesis class using data. The choice of hypothesis
class modulates a tradeoff between approximation error and estimation
error, as a small class can be learned with less data, but may have
worse asymptotic performance than a richer class. In the classical statistical learning setting, model selection
algorithms provide the following luckiness guarantee: If the class of
models decomposes as a nested sequence
$\Fcal_1 \subset \Fcal_2\subset\cdots\Fcal_m\subset\Fcal$, the sample
complexity of the algorithm scales with the statistical complexity of
the smallest subclass $\Fcal_{m^\star}$ containing the true model,
even though $m^\star$ is not known in advance. Such guarantees date back to Vapnik's structural risk
minimization principle and are by now well-known~\citep{vapnik1982estimation,vapnik1992principles,DevroyeLugosi96,birge1998minimum,shawe1998structural,lugosi1999adaptive,Koltchinskii2001,bartlett2002model,massart2007concentration}.
In practice, one may use cross-validation---the de-facto model selection procedure---to decide
whether to use, for example,  a linear model, a decision tree, or a neural network. That cross-validation appears in virtually every machine learning pipeline highlights the
necessity of model selection for successful ML deployments.

We investigate model selection in contextual bandits, a
simple interactive learning setting. Our main question is:
  \textit{Can model selection guarantees be achieved in contextual
    bandit learning, where a learner must balance exploration and
    exploitation to make decisions online?}

Contextual bandit learning is more challenging than
statistical learning on two fronts: First, decisions must be made
online without seeing the entire dataset, and second, the
learner's actions influence what data is observed (``bandit feedback''). 
Between these extremes is full-information online learning, where the learner does
not have to deal with bandit feedback, but still makes decisions online. Even in this simpler setting, model selection is challenging, since the learner must attempt to identify the appropriate model class while making
irrevocable decisions and incurring regret. Nevertheless, several prior
works on so-called parameter-free online learning \citep{mcmahan2013minimax,orabona2014simultaneous,koolen2015second,luo2015achieving,foster2017parameter,cutkosky2017online}
provide algorithms for online model selection with guarantees analogous to those in statistical learning. With bandit feedback, however, the learner must carefully balance exploration and exploitation, which presents a substantial challenge for model selection. At an intuitive level, the reason is that different hypothesis
classes require different amounts of exploration, but either over- or
under-exploring can incur significant regret (A detailed discussion requires a formal setup and is deferred
to~\pref{sec:preliminaries}).  At this point, it suffices to say that
prior to this work, we are not aware of any adequate model selection guarantee that
adapts results from statistical learning to any online learning setting with partial
information. 

We provide a new model selection
guarantee for the linear stochastic contextual
bandit setting~\citep{chu2011contextual,abbasi2011improved}.  We consider a sequence of
feature maps into $d_1 < d_2 < \ldots < d_M$ dimensions and assume
that the losses are linearly related to the contexts via the
$m^\star$-th feature map, so that the optimal policy is a
$d_{m^\star}$-dimensional linear policy. We design an
algorithm that achieves $\otil\prn{T^{2/3}d^{1/3}_{m^\star}}$
regret to this optimal policy over $T$ rounds, \emph{with no prior
  knowledge of $d_{\ms}$}. As this bound has no dependence on the maximum
dimensionality $d_M$, we say that the algorithm adapts to the
complexity of the optimal policy.  
All prior approaches suffer linear regret for non-trivial values of $d_{\ms}$, whereas the regret of our algorithm is sublinear whenever $d_{\ms}$ is such that the problem is
learnable. Our algorithm can also be tuned to achieve
$\tilde{O}\rbr{T^{3/4} + \sqrt{Td_{m^\star}}}$ regret, which matches
the optimal rate when $d_{\ms}\geq{}\sqrt{T}$.

At a technical level, we design a sequential test to determine
whether the optimal square loss for a large linear class is
substantially better than that of a smaller linear class. We show that this
test has \emph{sublinear} sample complexity: while learning a
near-optimal predictor in $d$ dimensions requires at least $\Omega(d)$ labeled examples,
we can estimate the improvement in \emph{value} of the optimal loss using only $O(\sqrt{d})$
examples, analogous to so-called variance estimation results in
statistics \citep{dicker2014variance,kong2018estimating}. Crucially, this
implies that we can test whether or not to
use the larger class without over-exploring for the smaller
class.


\section{Preliminaries}
\label{sec:preliminaries}

We work in the stochastic contextual bandit setting \citep{langford2008epoch,beygelzimer2011contextual,agarwal2014taming}. The setting is defined by a context space $\Xcal$, a finite action space $\Acal \defeq
\{1,\ldots,K\}$ and a distribution $\Dcal$ supported over $(x,\ell)$
pairs, where $x \in \Xcal$ and $\ell \in \RR^{\cA}$ is a loss
vector. The learner interacts with nature for $T$ rounds, where in
round $t$: (1) nature samples $(x_t,\ell_t) \sim \Dcal$, (2) the learner
observes $x_t$ and chooses action $a_t$, (3) the learner observes
$\ell_t(a_t)$. The goal of the learner is to choose actions to
minimize the cumulative loss. 

Following several prior works~\citep{chu2011contextual,abbasi2011improved,agarwal2012contextual,russo2013eluder,li2017provably},
we study a variant of the contextual bandit setting where the learner has access
to a class of regression functions
$\Fcal: \Xcal \times \Acal \to \RR$ containing the Bayes optimal regressor
\begin{equation}
\label{eq:realizable}
f^\star(x,a) \ldef \EE\sbr{\ell(a) \mid x}\quad \forall x,a.
\end{equation}
We refer to this assumption ($f^{\star}\in\cF$) as \emph{realizability}. For each regression function $f$
we define the induced policy $\pi_f(x) \ldef \argmin_{a} f(x,a)$. Note that $\pis\ldef{}\pi_{f^\star}$
is the globally optimal policy, and chooses the best action
on every context. We measure performance via regret to $\pis$:
\begin{align*}
\Reg \ldef \sum_{t=1}^T\ell_t(a_t) -
\sum_{t=1}^T\ell_t(\pis(x_t)).
\end{align*}
Low regret is tractable here due to the realizability assumption, and it is well known that the
optimal regret is $\tilde{\Theta}\rbr{\sqrt{T \cdot \textrm{comp}(\Fcal)}}$, where
$\textrm{comp}(\Fcal)$ measures the statistical complexity of $\Fcal$. For example, 
$\textrm{comp}(\Fcal) = \log |\Fcal|$ for finite classes, and $\textrm{comp}(\cF)=d$
for $d$-dimensional linear classes \citep{agarwal2012contextual,chu2011contextual}.\footnote{We suppress dependence on $K$ and logarithmic dependence on $T$ for
  this discussion.}

\paragraph{Model selection for contextual bandits.}
We aim for refined guarantees
that scale with the complexity of the optimal regressor $f^\star$
rather than the worst-case complexity of the class $\cF$. To this end, we assume that $\Fcal$ is
structured as a nested sequence of classes
$\Fcal_1\subset\Fcal_2\subset\ldots\subset\Fcal_M = \Fcal$, and we
define $m^\star \defeq \min\cbr{m: f^\star \in \Fcal_m}$. The model selection
problem for contextual bandits asks: \vspace{-0.5em}
\begin{center}
\emph{Given that $m^\star$ is not known in advance, can
  we achieve regret scaling as
  $\tilde{O}\prn{\sqrt{T\cdot \textnormal{comp}(\Fcal_{m^\star})}}$, rather than the less
  favorable $\tilde{O}\prn{\sqrt{T\cdot \textnormal{comp}(\Fcal)}}$?}
\end{center}\vspace{-.5em}
A slightly weaker model selection problem is to achieve $\tilde{O}\rbr{T^{\alpha}\cdot\textnormal{comp}(\Fcal_{m^\star})^{1-\alpha}}$ for some $\alpha \in
[\nicefrac{1}{2},1)$, again without knowing $m^\star$. Crucially, the exponents on
$T$ and $\textnormal{comp}(\Fcal_{m^\star})$ sum to one, implying that we can achieve
sublinear regret whenever $\textnormal{comp}(\Fcal_{m^\star})$ is sublinear in $T$, which is
precisely whenever the optimal model class is learnable. This implies that the bound, in spite of having worse dependence on $T$, adapts to the complexity of the optimal class with no prior knowledge.

We achieve this type of guarantee for linear contextual bandits. We assume that each regressor class
$\Fcal_m$ consists of linear functions of the form
\begin{align*}
\Fcal_m \defeq \cbr{(x,a) \mapsto \inner{\beta}{\phi^m(x,a)} \mid \beta \in \RR^{d_m}},
\end{align*}
where $\phi^m: \Xcal \times \Acal \to \RR^{d_m}$ is a fixed feature
map. To obtain a nested sequence of classes, and to ensure the
complexity is monotonically increasing, we assume that $d_1 < d_2 <
\ldots, d_M = d$ and that for each $m$, the feature map $\phi^m$ contains
the map $\phi^{m-1}$ as its first $d_{m-1}$ coordinates.\footnote{This
  is without loss of generality in a certain quantitative sense, since
  we can concatenate features without significantly
  increasing the complexity of
  $\Fcal_m$. See~\pref{cor:model_selection_non_nested}.} If
  $m^\star$ is the smallest feature map that realizes the optimal
  regressor, we can write \[f^\star(x,a)
  = \tri{\beta^\star,\phi^{m^\star}(x,a)},\]
 where $\beta^\star \in \RR^{d_{m^\star}}$ is the optimal coefficient
  vector.
In this setup, the optimal rate if $\ms$ is known is
$\Otilde(\sqrt{Td_{\ms}})$, obtained by \linucb{}
\citep{chu2011contextual}.\footnote{Regret scaling as $\tilde{O}(\sqrt{dT})$
  is optimal for the finite action setting we work in. Results for the
  infinite action case, where regret scales as
  $\tilde{\Theta}(d\sqrt{T})$, are incomparable to ours.} Our main
result achieves both $\otil\prn{T^{2/3}d^{1/3}_{m^\star}}$ regret
(i.e., $\alpha=2/3$) and $\otil\rbr{T^{3/4} +
  \sqrt{Td_{m^\star}}}$ regret without knowing $\ms$ in advance. 
\paragraph{Related work.}
The model selection guarantee we seek is straightforward
for full information online learning and statistical learning. A simple
strategy for the former setting is to use a low-regret online learner
for each sub-class $\Fcal_m$ and aggregate these base learners with a
master \textsf{Hedge} instance~\citep{freund1997decision}. Other
strategies include parameter-free methods
like \textsf{AdaNormalHedge}~\citep{luo2015achieving}
and \textsf{Squint}~\citep{koolen2015second}. Unfortunately, none of
these methods appear to adequately handle bandit feedback. For example, the
regret bounds of parameter-free methods do not depend on the so-called ``local
norms'', which are essential for achieving $\sqrt{T}$-regret in the
bandit setting via the usual importance weighting approach \citep{auer2002nonstochastic}. See~\pref{app:calculations} for further discussion.

In the bandit setting, two approaches we are aware of also fail:
the \textsf{Corral} algorithm of~\citet{agarwal2017corralling}, and an
adaptive version of the classical $\epsilon$-greedy
strategy~\citep{langford2008epoch}. Unfortunately, both algorithms
require tuning parameters in terms of the unknown index $m^\star$, and
naive tuning gives a guarantee of the form
$\tilde{O}(T^{\alpha}\textrm{comp}(\Fcal_{m^\star})^\beta)$ where $\alpha
+ \beta > 1$. For example, for finite classes \textsf{Corral} gives
regret $\sqrt{T} \log |\Fcal_{m^\star}|$. This guarantee is quite weak, since it is vacuous when $\log |\Fcal_{m^\star}| =
\Theta(\sqrt{T})$ even though such a class admits sublinear
regret if $\ms$ were known in advance (see \pref{app:calculations}). The conceptual takeaway from these
examples is that model selection for contextual bandits appears to
require new algorithmic ideas, even when we are satisfied with
$O\prn{T^{\alpha}\textrm{comp}(\Fcal_{m^\star})^{1-\alpha}}$-type rates.

Several recent papers have developed adaptive guarantees for
various contextual bandit settings. These include: (1)
adaptivity to easy data, where the optimal policy achieves low
loss~\citep{allenberg2006hannan,agarwal2017open,lykouris2017small,allen2018make},
(2) adaptivity to smoothness in settings with
continuous action
spaces~\citep{locatelli2018adaptivity,krishnamurthy2019contextual},
and (3) adaptivity in non-stationary environments, where distribution
drift parameters are
unknown~\citep{luo2017efficient,cheung2018learning,auer2018adaptively,chen2019new}. The
latter results can be cast as model selection with
appropriate nested classes of \emph{time-dependent} policies, but
these results are incomparable to our own, since they are specialized to the
non-stationary setting.

Interestingly, for multi-armed (non-contextual) bandits, several lower bounds
demonstrate that model selection is \emph{not} possible. The simplest
of these results is Lattimore's pareto
frontier~\citep{lattimore2015pareto}, which states that for
multi-armed bandits, if we want to ensure $O(\sqrt{T})$ regret against a single fixed arm
instead of the usual $O(\sqrt{KT})$ rate, we must incur $\Omega(K\sqrt{T})$ regret to the remaining $K-1$
arms. This precludes a model selection guarantee of the form
$\sqrt{T \cdot\textrm{comp}(\Acal)}$ since for bandits, the statistical
complexity is simply the number of arms.
Related lower bounds are known for Lipschitz
bandits~\citep{locatelli2018adaptivity,krishnamurthy2019contextual}. Our
results show that model selection \emph{is}
possible for contextual bandits, and thus highlight an important gap between the two settings.

In concurrent work,~\citet{chatterji2019osom} studied a similar model
selection problem with two classes, where the first class consists of
all $K$ constant policies and the second is a $d$-dimensional linear
class. They obtain logarithmic regret to the first class and
$O(\sqrt{Td})$ regret to the second, but their assumptions on the
context distribution are strictly stronger than our own. A detailed discussion is deferred
to the end of the section.

\paragraph{Technical preliminaries and assumptions.}
For a matrix $A$, $A^{\pinv}$ denotes the pseudoinverse and
$\nrm*{A}_{\specnorm}$ denotes the spectral norm. $I_{d}$ denotes the
identity matrix in $\bbR^{d\times{}d}$ and $\nrm*{\cdot}_{p}$ denotes
the $\ls_p$ norm. We use non-asymptotic big-$O$ notation, and use $\Otilde$ to hide terms logarithmic in $K$,
$d_{M}$, $M$, and $T$.

For a real-valued random variable $z$, we use the following notation
to indicate if $z$ is subgaussian or subexponential,
following \cite{vershynin2012introduction}:
\begin{equation}
z\sim\subG(\sigma^2) \Leftrightarrow \sup_{p\geq{}1}\crl{p^{-1/2}\prn{\En\abs*{z}^{p}}^{1/p}} \leq{}
\sigma,
\quad
z\sim\subE(\lambda) \Leftrightarrow \sup_{p\geq{}1}\crl{p^{-1}\prn{\En\abs*{z}^{p}}^{1/p}} \leq{}
\lambda.
\end{equation}
When $z$ is a random variable in $\bbR^{d}$, we write
$z\sim\subG_d(\sigma^{2})$ if $\tri*{\theta,z}\sim\subG(\sigma^{2})$
for all $\nrm*{\theta}_{2}=1$ and $z\sim\subE_d(\lambda)$ if
$\tri*{\theta,z}\sim\subE(\lambda)$ for all
$\nrm*{\theta}_{2}=1$. These definitions are equivalent to
many other familiar definitions for subgaussian/subexponential random
variables; see \pref{app:subgaussian}.

We assume that for each $m$ and  $a\in\cA$,
$\phi^m(x,a) \sim \subG(\tau_m^2)$ under $x\sim\cD$. Nestedness
implies that $\tau_1 \leq \tau_2\leq \ldots$, and we define $\tau=\tau_{M}$. We also assume that
$\ell(a) - \EE[\ell(a)\mid x] \sim \subG(\sigma^2)$ for all $x\in\cX$ and $a\in\cA$. Finally, we assume
that $\nrm*{\beta^\star} \leq B$. To keep notation clean, we use the convention that $\sigma\leq{}\tau$ and $B\leq{}1$, which ensures that
$\ls(a)\sim\subG(4\tau^{2})$.

We require a lower bound on the eigenvalues of the second
moment matrices for the feature
vectors.
For each $m$,
define
$\cov_m\defeq\frac{1}{K}\sum_{a\in\cA}\En_{x\sim\cD}\brk*{\phi^{m}(x,a)\phi^{m}(x,a)^{\trn}}$. We
let $\gamma_m^{2}\ldef\eigmin(\cov_m)$, where $\eigmin(\cdot)$ denotes
the smallest eigenvalue; nestedness implies
$\gamma_1\geq{}\gamma_2\geq{}\ldots$. We assume $\gamma_m\geq{}\gamma>0$ for all
$m$, and our regret bounds scale inversely proportional to $\gamma$.

Note that prior linear contextual bandit
algorithms \citep{chu2011contextual,abbasi2011improved} do not require
lower bounds on the second moment matrices. As discussed
earlier, the work of \cite{chatterji2019osom} obtains stronger model
selection guarantees in the case of two classes, but their result
requires a lower bound on
$\eigmin\prn*{\En\brk*{\phi(x,a)\phi(x,a)^{\top}}}$ for all
actions. Previous work suggests that advanced exploration is not
needed under such assumptions \citep{bastani2017mostly, kannan2018smoothed,
  raghavan2018externalities}, which considerably simplifies the
problem.\footnote{It appears that exploration is still required for linear contextual bandits under our average
  eigenvalue assumption. Consider the case $d=2$ and $\beta^\star =
  (1/2,1)$. Suppose there are four actions, and that at the first
  round, $\phi(x,\cdot)=\crl*{ e_1, -e_1, e_2, -e_2}$. We can ensure  that with probability $1/2$, the
  first action played will be one of the first two. At this point a
  greedy strategy will always choose $e_1$, but the average context distribution has minimum eigenvalue $1$. } As such, the result should be seen as complementary to our
own. Whether model selection can be achieved without some type of eigenvalue condition is an important open question.


\section{Model selection for linear contextual bandits}
\label{sec:model_selection_main}

We now present our algorithm for model selection in
linear contextual bandits, \mainalgo (``\mainalgolong{}''). Pseudocode
is displayed in~\pref{alg:main}. The algorithm maintains an ``active'' policy class index $\mh \in [M]$, which it updates
over the $T$ rounds starting from $\mh=1$.  The algorithm updates $\mh$
only when it can prove that $\mh \ne m^\star$, which is achieved
through a statistical test called \estimateresidual{} (\pref{alg:estimator}). When $\mh$ is not being updated, the algorithm
operates as if $\mh=m^\star$ by running a low-regret contextual bandit
algorithm with the policies induced by $\Fcal_{\mh}$; we call this
policy class $\Pi_m
\defeq \cbr{x \mapsto \argmin_{a \in \Acal} \inner{\beta}{\phi^m(x,a)} \mid
  \nrm*{\beta}_2 \leq \tau/\gamma}$.\footnote{The norm constraint $\tau/\gamma$ guarantees that $\Pi_m$ contains parameter vectors
  arising from a certain square loss minimization problem under our
  assumption that $\nrm*{\betas}_{2}\leq{}1$; see \pref{prop:beta_basics}.}
Note that the low-regret algorithm we run for $\Pi_{\mh}$ cannot based on
realizability, since $\Fcal_{\mh}$ will
not contain the true regressor $f^{\star}$ until we reach $\ms$. This rules out the usual
linear contextual bandit algorithms such as \linucb{}. Instead we use a variant
of \expix~\citep{neu2015explore}, which is an agnostic
contextual bandit algorithm and does not depend on realizability. To deal with infinite classes, unbounded losses, and other
technical issues, we require some simple modifications to
\expix; pseudocode and analysis are deferred to~\pref{app:cb}.

\begin{algorithm}[t]
\begin{algorithmic}
\STATE \textbf{input:}
\begin{itemize}[leftmargin=2em]
\item Feature maps $\crl*{\phi^{m}(\cdot,\cdot)}_{m\in\brk{M}}$, where
  $\phi^{m}(x,a)\in\bbR^{d_m}$, and time $T\in\bbN$.
\item Subgaussian parameter $\tau>0$, second moment parameter $\gamma>0$.
\item Failure probability $\delta\in(0,1)$, exploration parameter $\expexp\in(0,1)$, confidence parameters. $C_1,C_2>0$.
\end{itemize}
\STATE \textbf{definitions:}
\begin{itemize}[leftmargin=2em]
\item Define $\delta_0=\delta/10M^{2}T^2$ and $\mu_{t}=(K/t)^{\expexp}\wedge{}1$.
\item Define $\alpha_{m,t} = C_1\cdot\Bigl(\frac{\tau^{6}}{\gamma^{4}}\cdot\frac{d_m^{1/2}\log^{2}(2d_m/\delta_0)}{K^{\expexp}t^{1-\expexp}}
    + \frac{\tau^{10}}{\gamma^{8}}\cdot{}\frac{d_m\log(2/\delta_0)}{t}\Bigr)$. 
\item Define $T_{m}^{\mathrm{min}} =
  C_2\cdot{}\prn*{\frac{\tau^{4}}{\gamma^{2}}\cdot{}d_m\log(2/\delta_0)+\log^{\frac{1}{1-\kappa}}(2/\delta_0)
  + K}+1$.
\end{itemize}
\STATE \textbf{initialization:}
\begin{itemize}[leftmargin=2em]
\item $\mh\gets{}1$. \algcomment{Index of candidate policy class.}
\item $\currentalg_{1}\gets\expix(\Picompact_1, T, \delta_0)$.
\item $S\gets\crl*{\emptyset}$. \algcomment{Times at which uniform
    exploration takes place.}
\end{itemize}
\FOR{$t=1,\ldots,T$}
\STATE Receive $x_t$.
\STATE \textbf{with probability $1-\mu_t$}
\STATE ~~~~~~Feed $x_t$ into $\currentalg_{\mh}$ and
take $a_t$ to be the predicted action. 
\STATE ~~~~~~Update $\currentalg_{\mh}$ with
$(x_t, a_t, \ls_t(a_t))$.
\STATE \textbf{otherwise}
\STATE ~~~~~~Sample $a_t$ uniformly from $\cA$ and let $S\gets{}S\cup\crl*{t}$.
\STATE \algcommentbig{Test whether we should move on to a larger policy
  class.}
\STATE $\empcov_{i}\gets\frac{1}{K}\sum_{a\in\cA}\sum_{s=1}^{t}\phi^{i}(x_{s},a)
\phi^{i}(x_{s},a)^{\trn}$ for each $i\geq\mh$. \algcomment{Empirical
  second moment.}
\STATE $H_i \gets \crl{(\phi^{i}(x_{s},a_{s}),\ls(a_s))}_{s\in{}S}$ for each $i>\mh$.
\STATE $\wh{\cE}_{\mh,i}\gets\estimateresidual\prn*{H_i,
  \empcov_{\mh}, \empcov_{i}}$ for each $i>\mh$. \algcomment{Gap estimate.}

\IF{there exists $i>\mh$ such that
  $\wh{\cE}_{\mh,i}\geq{}2\alpha_{i,t}$ and $t\geq{}\Tmin_i$}
\STATE Let $\mh$ be the smallest such $i$. Re-initialize $\currentalg_{\mh}\gets\expix(\Picompact_{\mh}, T, \delta_{0})$.
\ENDIF
\ENDFOR
\end{algorithmic}
\caption{\mainalgo (\mainalgolong)}
\label{alg:main}
\end{algorithm}

\subsection{Key idea: Estimating prediction error with sublinear
  sample complexity}
Before stating the main result, we elaborate on the statistical test
(\estimateresidual{}) used in \pref{alg:main}. \estimateresidual{} estimates an
upper bound on the gap between the best-in-class loss for two policy
classes $\Pi_i$ and $\Pi_j$, which we define as $\Delta_{i,j} \defeq
\Ls_{i}-\Ls_j$, where $\Ls_i\defeq\min_{\pi\in\Pi_i}L(\pi)$. At each round, \pref{alg:main} uses \estimateresidual{} to estimate the gap $\Delta_{\mh,i}$ for all $i>\mh$. If the estimated
gap is sufficiently large for some $i$, the algorithm sets $\mh$ to the smallest
such $i$ for the
next round. This approach is based on the following
observation: For all $m\geq\ms$, $\Ls_{m}=\Ls_{\ms}$. Hence, if
$\Delta_{\mh,i}>0$, it must be the case that $\mh\neq{}\ms$, and we
should move on to a larger class.

The key challenge is to estimate $\Delta_{\mh,i}$ while ensuring low
regret. Naively, we could use uniform exploration and find a policy in
$\Pi_i$ that has minimal empirical loss, which gives an estimate of
$\Ls_i$.  Unfortunately, this requires tuning the exploration
parameter in terms of $d_i$ and would compromise the regret if $\mh
=\ms$. Similar tuning issues arise with other approaches and are discussed further in \pref{app:calculations}.

We do not estimate the gaps $\Delta_{i,j}$ directly, but instead estimate an upper bound motivated by the realizability
assumption. For each $m$, define
 \begin{equation}
   \label{eq:square_loss_min}
\betas_m\defeq\argmin_{\beta\in\bbR^{d_m}}\frac{1}{K}\sum_{a\in\cA}\En_{x\sim\cD}\prn*{\tri{\beta,\phi^{m}(x,a)}-\ls(a)}^{2},
 \end{equation}
and define\footnote{In
    \pref{app:main} we show that $\betas_m$ and consequently
    $\cE_{i,j}$ are always uniquely defined.}
\begin{equation}
  \label{eq:square_loss_gap}
  \cE_{i,j}\defeq\frac{1}{K}\sum_{a\in\cA}\En_{x\sim\cD}\prn*{\tri{\betas_{i},\phi^{i}(x,a)}-\tri{\betas_j,\phi^{j}(x,a)}}^{2}.
\end{equation}
We call $\cE_{i,j}$ the \emph{square loss gap} and we
call $\Delta_{i,j}$ the \emph{policy gap}. A key lemma driving these
definitions is that the square loss gap upper bounds the policy gap.
\begin{lemma}
\label{lem:sq_translation}
For all $i \in [M]$ and $j \geq \ms$, $\Delta_{i,j} \leq
\sqrt{4K\cE_{i,j}}$. Furthermore, if $i,j \geq \ms$ then $\cE_{i,j} = 0$.
\end{lemma}

With realizability, the square loss gap behaves similar to the
policy gap: it is non-zero whenever the latter is non-zero, and $\ms$ has zero gap to all $m \geq \ms$. Therefore, we seek estimators for the square loss gap
$\cE_{\mh,i}$ for $i>\mh$. Observe that $\cE_{\mh,i}$ depends on the optimal predictors
$\betas_{\mh},\betas_i$ in the two classes. A natural approach to estimate
$\cE_{\mh,i}$ is to solve regression problems over both classes to
estimate the predictors, then plug them into the expression for $\cE$;
we call this the \emph{plug-in approach}. As this relies on linear
regression, it gives an $O(\nicefrac{d_i}{n})$ error rate for estimating
$\cE_{\mh,i}$ from $n$ uniform exploration samples. Unfortunately, since the error scales linearly with the size of the
larger class, we must over-explore to ensure low error, and this
compromises the regret if the smaller class is optimal.

As a key technical contribution, we design more efficient estimators
for the square loss gap $\cE_{\mh,i}$. We work in the following slightly more general gap estimation setup: we receive pairs $\{(x_s,y_s)\}^n_{s=1}$ i.i.d. from a distribution $\Dcal \in \Delta(\RR^d \times \RR)$, where
$x \sim \subG(\tau^2)$ and $y \sim \subG(\sigma^2)$. We partition the feature space into $x = (x^{(1)}, x^{(2)})$, where $x^{(1)} \in \RR^{d_1}$, and define
\begin{align*}
\beta^\star \defeq \argmin_{\beta \in \RR^d} \EE\rbr{\inner{\beta}{x}-y}^2, \qquad \beta_1^\star \defeq \argmin_{\beta \in \RR^{d_1}} \EE \rbr{\tri{\beta,x^{(1)}} - y}^2.
\end{align*}
These are, respectively, the optimal linear predictor and the optimal
linear predictor restricted to the first $d_1$ dimensions. The square loss gap for the two predictors is defined as
$\cE \defeq \EE\rbr{\inner{\beta^\star}{x}
- \tri{\beta_1^\star,x^{(1)}}}^2$.  Our problem of estimating
$\cE_{\mh,i}$ clearly falls into this general setup if we uniformly
explore the actions for $n$ rounds, then set $\{x_s\}^n_{s=1}$ to be
the features obtained through the feature map $\phi^i$ and
$\{y_s\}^n_{s=1}$ to be the observed losses.

\begin{algorithm}[t]
\begin{algorithmic}
\STATE \textbf{input:} Examples
$\{(x_s,y_s)\}_{s=1}^n$, second moment matrix estimates
$\empcov\in\bbR^{d\times{}d}$ and $\empcov_1\in\bbR^{d_1\times{}d_1}$.
\STATE Define $d_2 = d - d_1$ and
\[
\wh{R}=\wh{D}^{\pinv}-\empcov^{\pinv},\quad\text{where}\quad
\wh{D} = \left(\begin{array}{ll}
\empcov_1&0_{d_1\times{}d_2}\\
0_{d_2\times{}d_1}&0_{d_2\times{}d_2}
\end{array}
\right).
\]
\STATE Return estimator
\[
\wh{\cE}=\frac{1}{{n\choose2}}\sum_{s<t}\tri*{\empcov^{1/2}\wh{R}x_s y_s,\empcov^{1/2}\wh{R}x_t y_t}.
\]
\end{algorithmic}
\caption{\estimateresidual}
\label{alg:estimator}
\end{algorithm}

The pseudocode for our estimator \estimateresidual is displayed in~\pref{alg:estimator}.
In addition to the $n$ labeled samples, it takes as input two empirical second moment matrices $\empcov$ and $\empcov_1$ constructed via an extra set of $m$ i.i.d. unlabeled samples; these serve as estimates for
$\Sigma \defeq \EE \sbr{x x^\top}$ and $\Sigma_1 \defeq \EE\sbr{x^{(1)}x^{(1)\top}}$.
The intuition 
is that one can write
the square loss gap as $\cE=\nrm*{\Sigma^{1/2}R\En\brk*{xy}}_{2}^{2}$
where $R\in\bbR^{d\times{}d}$ is a certain function of $\Sigma$ and $\Sigma_1$.
\estimateresidual simply replaces the second moment matrices
with their empirical counterparts and then uses the labeled examples to
estimate the weighted norm of $\En\brk*{xy}$ through a U-statistic.
The main guarantee for the estimator is as follows.
\begin{theorem}
\label{thm:residual_est_main}
Suppose we take $\empcov$ and $\empcov_1$ to be the empirical
second moment matrices formed from $m$ i.i.d. unlabeled samples. Then when $m \geq
C(d + \log(2/\delta))\tau^4/\eigmin(\cov)$, \estimateresidual, given $n$ labeled samples, guarantees that with probability at least $1-\delta$,
\begin{align}
\label{eq:residual_est_main}
  \abs*{\wh{\cE}-\cE} \leq{} \frac{1}{2}\cE 
+
O\prn*{\frac{\sigma^{2}\tau^{4}}{\eigmin^{2}(\cov)}\cdot\frac{d^{1/2}\log^{2}(2d/\delta)}{n}
+ 
\frac{\tau^{6}}{\eigmin^{4}(\cov)}\cdot\frac{d\log(2/\delta)}{m}\cdot \nrm*{\En\brk*{xy}}_{2}^{2}
}.
\end{align}
\end{theorem}
To compare with the plug-in approach, we focus on the dependence
between $d$ and $n$. When \estimateresidual is applied within \mainalgo{} we have plenty of
unlabeled data, so the dependence on $m$ is less important. The
dominant term in~\pref{thm:residual_est_main} is
$\tilde{O}(\nicefrac{\sqrt{d}}{n})$, a significant improvement over
the $\tilde{O}(\nicefrac{d}{n})$ rate for the plug-in
estimator. In particular, the dependence on the larger ambient
dimension is much
milder: we can achieve constant error with $n \asymp \sqrt{d}$, or in
other words the estimator has \emph{sublinear} sample complexity. This property is crucial for our model selection result. The result
generalizes and is inspired by the variance estimation method of Dicker
\citep{dicker2014variance,kong2018estimating}, which obtains a rate of
$O\prn*{\nicefrac{\sqrt{d}}{n}+\nicefrac{1}{\sqrt{n}}}$ to estimate the optimal
square loss $\min_{\beta\in\bbR^{d}}\En\prn*{\tri*{\beta,x}-y}^{2}$
when the second moments are known. By estimating the square
loss \emph{gap} instead of the loss itself, we avoid the
$\nicefrac{1}{\sqrt{n}}$ term, which is critical for achieving
$\tilde{O}(T^{2/3}d_{\ms}^{1/3})$ regret.

\subsection{Main result}

Equipped with \estimateresidual, we can now sketch the approach behind \mainalgo in a bit more detail. Recall that the algorithm maintains an index $\mh$
denoting the current guess for $m^\star$. We run \expix\xspace over
the induced policy class $\Pi_m$, mixing in some additional uniform
exploration (with probability $\mu_t$ at round $t$). We use all of the data to estimate the second moment
matrices of all classes, and we pass only the exploration data into
the subroutine \estimateresidual. We check if there exists some $i > \mh$ such that the estimated gap satisfies $\wh{\cE}_{\mh,i}\geq{}2\alpha_{i,t}$ and $t\geq{}\Tmin_i$ which---based on the deviation bound in~\pref{thm:residual_est_main}---implies that $\cE_{\mh,i} > 0$ and thus $\mh \neq \ms$.
If this is the case, we advance $\mh$ to the smallest such $i$, and if not, we
continue with our current guess. This leads to the following guarantee.
\begin{theorem}
\label{thm:model_selection_main}
When $C_1$ and $C_2$ are sufficiently large absolute constants and
$\kappa=1/3$, \mainalgo{} guarantees that with probability at least
$1-\delta$,
\begin{equation}
  \label{eq:mainbound1}
 \Reg \leq{} \Otilde\prn*{
  \frac{\tau^{4}}{\gamma^{3}}\cdot
  (T\ms)^{2/3}(Kd_{\ms})^{1/3}\log(2/\delta)}.  
\end{equation}
When $\kappa=1/4$, \mainalgo{} guarantees that with probability at
least $1-\delta$,
\begin{equation}
  \label{eq:mainbound2}
\Reg
\leq {}
\Otilde\prn*{
\frac{\tau^{3}}{\gamma^{2}}\cdot{}K^{1/4}(T\ms)^{3/4}\log(2/\delta)
+ \frac{\tau^{5}}{\gamma^{4}}\cdot{}\sqrt{K(T\ms)d_{\ms}}\log(2/\delta)}.  
\end{equation}
\end{theorem}
A few remarks are in order
\begin{itemize}
\item 
The two stated bounds are incomparable in general. Tracking only
dependence on $T$ and $d_{m^\star}$, the first is
$\tilde{O}(T^{2/3}d_{m^\star}^{1/3})$ while the latter is
$\tilde{O}(T^{3/4} + \sqrt{Td_{m^\star}})$. The former is better when
$d_{m^\star} \leq T^{1/4}$. There is a more general Pareto frontier
that can be explored by choosing $\kappa\in[1/3,1/4]$, but no choice for
$\kappa$ dominates the others for all values of $d_{\ms}$.
\item 
Recall that if had we known $d_{m^\star}$ in advance, we have could simply
run \linucb to achieve $\tilde{O}\prn{\sqrt{Td_{m^\star}}}$
regret. The bound~\pref{eq:mainbound2} matches this oracle rate when
$d_{m^\star} > \sqrt{T}$, but otherwise our guarantee is slightly
worse than the oracle rate. Nevertheless, both bounds are $o(T)$
whenever the oracle rate is $o(T)$ (that is, when $d_{m^\star} = o(T)$), so the algorithm
has sublinear regret whenever the optimal model class
is \emph{learnable}. It remains open whether there is a model
selection algorithm that can match the oracle rate for all values of $d_{m^{\star}}$ simultaneously.
\item We have not optimized dependence on the condition
  number $\tau/\gamma$ or the logarithmic factors.
\item If the individual distribution parameters
  $\crl*{\tau_{m}}_{m\in\brk*{M}}$ and
  $\crl*{\gamma_{m}}_{m\in\brk*{M}}$ are known, the algorithm can be modified
  slightly so that regret scales in terms of $\tau_{\ms}$ and
  $\gamma_{\ms}^{-1}$. However the current model, in which we assume access only to uniform upper and lower bounds on these parameters, is more realistic.
\end{itemize}
\ifthenelse{\equal{\version}{neurips}}{Additionally, \pref{app:experiment} contains a validation experiment, where we demonstrate that \mainalgo{} compares favorably to \linucb{} in simulations.}{}

\paragraph{Improving the dependence on $\ms$.}
\pref{thm:model_selection_main} obtains the desired model selection
guarantee for linear classes, but the bound includes a polynomial
dependence on the optimal index $\ms$. This contrasts the logarithmic
dependence on $\ms$ that can be obtained through structural risk
minimization in statistical learning
\citep{vapnik1992principles}. However, this $\mathrm{poly}(\ms)$
dependence can be replaced by a single $\log(T)$ factor with a simple preprocessing step:
Given feature maps $\crl*{\phi^{m}(\cdot,\cdot)}_{m\in\brk*{M}}$ we
construct a new collection of maps
$\crl*{\bar{\phi}^{m}(\cdot,\cdot)}_{m\in\brk*{\bar{M}}}$, where
$\bar{M}\leq{}\log{}T$, as follows.
First, for $i=1,\ldots,\log{}T$, take $\bar{\phi}^{i}$ to be the largest feature map $\phi^{m}$ for which $d_{m}\leq{}e^{i}$.
Second, remove any duplicates.
This preprocessing reduces the number of feature maps to at most
$\log(T)$ while ensuring that a map of dimension $O(d_{\ms})$ that contains $\phi^{\ms}$ is always available. Specifically, the preprocessing step yields the following improved regret bounds.
\begin{theorem}
\label{thm:model_selection_log}
\mainalgo{} with preprocessing guarantees that
with probability at least $1-\delta$,
\begin{align*}
 \Reg \leq{} \left\{\begin{array}{ll}
\Otilde\prn*{
  \frac{\tau^{4}}{\gamma^{3}}\cdot
  T^{2/3}(Kd_{\ms})^{1/3}\log(2/\delta)},\quad\quad&\kappa=1/3.\\
~&\\
\Otilde\prn*{
\frac{\tau^{3}}{\gamma^{2}}\cdot{}K^{1/4}T^{3/4}\log(2/\delta)
+
  \frac{\tau^{5}}{\gamma^{4}}\cdot{}\sqrt{KTd_{\ms}}\log(2/\delta)},\quad\quad&\kappa=1/4.
\end{array}\right. 
\end{align*}
\end{theorem}

\paragraph{Non-nested feature maps.}
As a final variant, we note that the algorithm easily extends to the
case where feature maps are not nested. Suppose we have non-nested
feature maps $\crl*{\phi^{m}}_{m\in\brk*{M}}$, where
$d_1\leq{}d_2\leq{}\ldots\leq{}d_{M}$; note that the inequality is no
longer strict. In this case, we can obtain a nested collection by
concatenating $\phi^{1},\ldots,\phi^{m-1}$ to the map $\phi^{m}$ for
each $m$. This process increases the dimension of the optimal map from
$d_{\ms}$ to at most $d_{\ms}\ms$, so we have the following corollary.
\begin{corollary}
\label{cor:model_selection_non_nested}
For non-nested feature maps, \mainalgo{} with preprocessing guarantees that
with probability at least $1-\delta$,
\begin{align*}
 \Reg \leq{} \left\{\begin{array}{ll}
\Otilde\prn*{
  \frac{\tau^{4}}{\gamma^{3}}\cdot
  T^{2/3}(Kd_{\ms}\ms)^{1/3}\log(2/\delta)},\quad\quad&\kappa=1/3.\\
~&\\
\Otilde\prn*{
\frac{\tau^{3}}{\gamma^{2}}\cdot{}K^{1/4}T^{3/4}\log(2/\delta)
+
  \frac{\tau^{5}}{\gamma^{4}}\cdot{}\sqrt{KTd_{\ms}\ms}\log(2/\delta)},\quad\quad&\kappa=1/4.
\end{array}\right. 
\end{align*}
\end{corollary}


\subsection{Proof sketch}
\label{sec:sketches}
We now sketch the proof of the main theorem, with the full proof
deferred to \pref{app:main}. We focus on the case where there are just two classes, so $M=2$. We only
track dependence on $T$ and $d_m$, as this preserves the relevant
details but simplifies the argument. The analysis has two cases depending
on whether $f^{\star}$ belongs to $\cF_1$ or $\cF_2$.

First, if $f^\star \in \cF_1$ then
by~\pref{lem:sq_translation} we have that $\cE_{1,2} = 0$. Further,
via \pref{thm:residual_est_main}, we can guarantee that we never
advance to $\wh{m}=2$ with high probability. The result then
follows from the regret guarantee
for \expix{} using policy class $\Pi_1$, and by accounting for uniform exploration.

The more challenging case is when $f^\star \in \cF_2$. Let $\wh{T}$
denote the first round where $\wh{m}=2$, or $T$ if the algorithm
never advances. We can bound regret as
\begin{align*}
\Reg \leq O\rbr{T^{1-\kappa}} + \tilde{O}\rbr{\sqrt{\wh{T}d_1}} + \wh{T}\Delta_{1,2} + \tilde{O}\rbr{\sqrt{(T-\wh{T})d_2}}.
\end{align*}
The four terms correspond to: (1) uniform exploration with probability
$\mu_t \asymp t^{-\kappa}$ in round $t$, (2) the \expix\xspace regret bound
for class $\Pi_1$ until time $\wh{T}$, (3) the policy gap between the
best policy in $\Pi_1$ and the optimal policy $\pi^\star \in \Pi_2$,
and (4) the \expix\xspace bound over class $\Pi_2$ until time $T$. The two
regret bounds (the second and fourth terms) clearly contribute
$O(\sqrt{Td_2})$ to the overall regret, and the first term is controlled by our choice of $\kappa\in\crl*{1/4, 1/3}$. We are left to bound the third term. For this term, observe that in round $\wh{T}-1$, since the algorithm did
not advance, we must have $\wh{\cE}_{1,2} \leq
2\alpha_{2,\wh{T}-1}$. Appealing to~\pref{thm:residual_est_main}, this
implies that $\cE_{1,2} \leq O(\wh{\cE}_{1,2}+\alpha_{2,\wh{T}-1})\leq
O(\alpha_{2,\wh{T}-1})$. Plugging in the definition of $\alpha_{2,t}$
and applying~\pref{lem:sq_translation}, this gives
\begin{align}
\wh{T}\Delta_{1,2} \leq \order\rbr{\wh{T}\sqrt{\Ecal_{1,2}}} \leq \order\rbr{\wh{T}\sqrt{\alpha_{2,\wh{T}-1}}} = \order\rbr{T^{\frac{1+\kappa}{2}}d^{1/4}_2}.\label{eq:proof_sketch_gap}
\end{align}
This establishes the result. In particular, with $\kappa=1/3$ we obtain
$\tilde{O}(T^{2/3}+\sqrt{Td_2} + T^{2/3}d_2^{1/4}) \leq \tilde{O}(T^{2/3}d_2^{1/3})$ regret, and
with $\kappa = 1/4$ we obtain $\tilde{O}(T^{3/4} + \sqrt{Td_2} + T^{5/8}d_2^{1/4}) =
\tilde{O}(T^{3/4} + \sqrt{Td_2})$.\footnote{Note that if $d_2\leq{}\sqrt{T}$
  then $T^{5/8}d_2^{1/4}\leq{}T^{3/4}$, but if $d_2\geq\sqrt{T}$ then $T^{5/8}d_2^{1/4}\leq{}\sqrt{Td_2}$.}

The sublinear estimation rate for \estimateresidual (\pref{thm:residual_est_main}) plays a critical role in this argument.
With the $\tilde{O}(\nicefrac{d}{n})$ rate for the na\"{i}ve plug-in estimator, we could
at best set $\alpha_{t,2} = \nicefrac{d_2}{t^{1-\kappa}}$, but this
would degrade the dimension dependence in \pref{eq:proof_sketch_gap} from $d_2^{1/4}$ to
$\sqrt{d_2}$. Unfortunately, this results in $T^{1-\kappa}
+ \sqrt{d_2}T^{\frac{1+\kappa}{2}}$ regret, which is not a meaningful
model selection result, since there is no choice for $\kappa\in(0,1)$ for
which the exponents on $d_2$ and $T$ sum to one for both terms.


\section{A validation experiment}
\label{sec:experiment}
As an empirical validation, we conducted a synthetic experiment with
\mainalgo. Our implementation is built on top of an open source package
for contextual bandit experimentation, which has been used in several
prior
works~\citep{krishnamurthy2016contextual,foster2018practical,krishnamurthy2018semiparametric}.\footnote{Code
  is publicly available at
  \url{https://github.com/akshaykr/oracle_cb}.} For computational
efficiency, our implementation of \mainalgo uses \iltcb (\cite{agarwal2014taming}, henceforth ``\iltcbshort'') as the base
learner instead of \currentalg, which is also sufficient for our
theoretical guarantees.

\begin{wrapfigure}{r}{0.4\textwidth}
\includegraphics[width=0.4\textwidth]{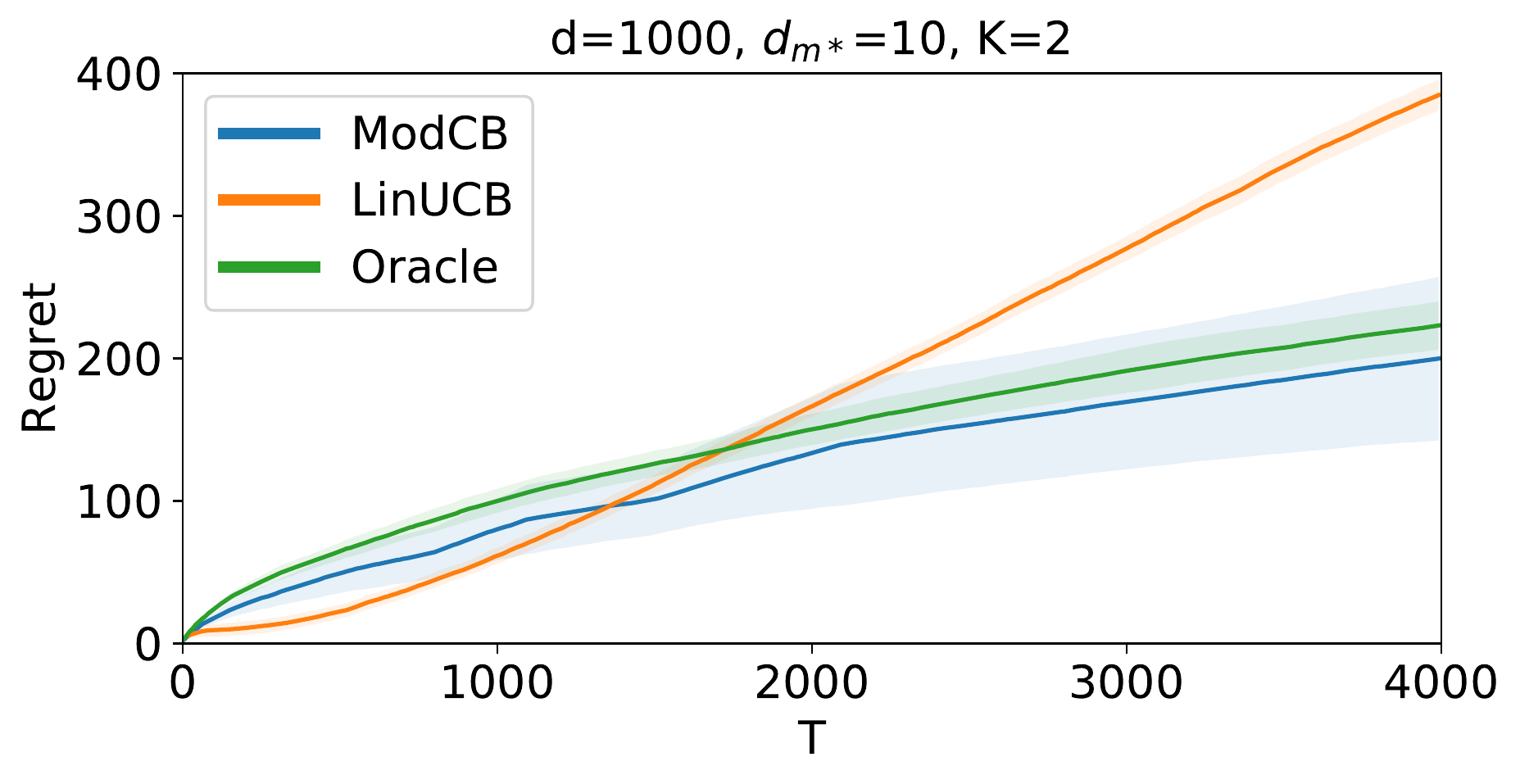}
\vspace{-0.75cm}
\caption{A validation experiment}
\label{fig:experiment}
\end{wrapfigure}
We consider a simple contextual bandit environment with $K=2$ actions,
and where the feature vectors have ambient dimension $d=1000$. We design the reward distribution such that a predictor using the first $d_{m^\star} = 10$
coordinates is realizable. We consider three algorithms: \linucb
operating directly on the ambient dimension, \mainalgo with \iltcbshort as
the base learner, and \iltcbshort with knowledge of $d_{m^\star}$, which we
refer to as \oracle. We tune a single hyperparameter for each
algorithm (the confidence pre-multiplier for \linucb; the uniform
exploration parameter for \mainalgo and \oracle) and visualize the
cumulative regret for the best performing hyperparameter as a
function of the number of rounds $T$. \pref{fig:experiment} displays
the regret curves averaged over 20 replicates with error bands
corresponding to two standard errors.

\mainalgo consistently outperforms \linucb in the experiment, which is
unsurprising since the ambient dimension is much larger than the
target dimension $d_{m^\star}$. \mainalgo has a less favorable dependence on $T$ in comparison to
\linucb, but this does not seem to compromise its performance in this
experiment. Perhaps more surprising is that \mainalgo
outperforms \oracle on average. A deeper inspection reveals that
while \mainalgo typically advances to $d_m > d_{m^\star}$, it sometimes
stays below, where it can learn a near-optimal policy faster
than \oracle.

\section{Discussion}
\label{sec:discussion}

This paper initiates the study of model selection tradeoffs in
contextual bandits. We provide the first model selection algorithm for
the linear contextual bandit setting, which we show achieves
$\tilde{O}\rbr{T^{2/3}d_{m^\star}^{1/3}}$ when the optimal model is a
$d_{m^\star}$-dimensional linear function. This is the first contextual bandit algorithm that adapts the complexity of the optimal policy class with no prior knowledge, and raises a number of intriguing questions:
\begin{enumerate}
\vspace{-5pt}
\item Is it possible to achieve $\tilde{O}\rbr{\sqrt{Td_{m^\star}}}$ regret in
  our setup? This would show that the price for model selection is
  negligible.
\item Is it possible to generalize our results beyond linear classes?
  Specifically, given regressor classes
  $\cF_1\subset\cF_2\subset\ldots\subset\cF_M$ and assuming the
  optimal model $f^\star$ belongs to $\cF_{m^\star}$ for some unknown
  $m^\star$, is there a contextual bandit algorithm that achieve
  $\tilde{O}\rbr{T^{\alpha} \cdot \textrm{comp}^{1-\alpha}(\cF_{m^\star})}$ regret for
  some $\alpha \in [\nicefrac{1}{2},1)$? More ambitiously, can we
  achieve $\tilde{O}\rbr{\sqrt{T\cdot\textrm{comp}(\cF_{m^\star})}}$? 
\item \ifthenelse{\equal{\version}{neurips}}{%
We have conducted a validation experiment with \mainalgo
(see~\pref{app:experiment}), demonstrating that the algorithm
performs favorably in simulations.}{}
While this synthetic experiment is
  encouraging, \mainalgo may not be immediately useful for practical
  deployments for several reasons, including its reliance on linear
  realizability and its computational complexity. Are there more robust
  algorithmic principles for theoretically-sound and
  practically-effective model selection in contextual bandits?
\item For what classes $\cF$ can we estimate the loss at a
sublinear rate, and is this necessary for contextual bandit model selection? Any sublinear guarantee will lead to non-trivial
model selection guarantees through a strategy similar to
\mainalgo{}. Interestingly, it is already known that for certain
(e.g., sparse linear) classes, sublinear loss estimation is not
possible \citep{verzelen2018adaptive}. On the other hand, positive
results \emph{are} available for certain nonparametric classes \citep{brown2007variance,wang2008effect}.
\vspace{-5pt}
\end{enumerate}
Model selection is instrumental to the success of
ML deployments, yet few positive results exist for partial feedback
settings. We believe these questions are technically challenging and practically
important, and we are hopeful that positive results of the type we provide here will
extend to interactive learning settings beyond contextual bandits. 


\subsection*{Acknowledgements}
We thank Ruihao Zhu for working with us at the early stages of this
project, and for many helpful discussions. AK is supported in part by NSF Award IIS-1763618.
HL is supported by NSF IIS-1755781.

\bibliography{refs}

\begin{thebibliography}{56}
\providecommand{\natexlab}[1]{#1}
\providecommand{\url}[1]{\texttt{#1}}
\expandafter\ifx\csname urlstyle\endcsname\relax
  \providecommand{\doi}[1]{doi: #1}\else
  \providecommand{\doi}{doi: \begingroup \urlstyle{rm}\Url}\fi

\bibitem[Abbasi-Yadkori et~al.(2011)Abbasi-Yadkori, P{\'a}l, and
  Szepesv{\'a}ri]{abbasi2011improved}
Yasin Abbasi-Yadkori, D{\'a}vid P{\'a}l, and Csaba Szepesv{\'a}ri.
\newblock Improved algorithms for linear stochastic bandits.
\newblock In \emph{Advances in Neural Information Processing Systems}, 2011.

\bibitem[Agarwal et~al.(2012)Agarwal, Dud{\'\i}k, Kale, Langford, and
  Schapire]{agarwal2012contextual}
Alekh Agarwal, Miroslav Dud{\'\i}k, Satyen Kale, John Langford, and Robert~E
  Schapire.
\newblock Contextual bandit learning with predictable rewards.
\newblock In \emph{International Conference on Artificial Intelligence and
  Statistics}, 2012.

\bibitem[Agarwal et~al.(2014)Agarwal, Hsu, Kale, Langford, Li, and
  Schapire]{agarwal2014taming}
Alekh Agarwal, Daniel Hsu, Satyen Kale, John Langford, Lihon Li, and Robert~E.
  Schapire.
\newblock Taming the monster: A fast and simple algorithm for contextual
  bandits.
\newblock In \emph{International Conference on Machine Learning}, 2014.

\bibitem[Agarwal et~al.(2017{\natexlab{a}})Agarwal, Krishnamurthy, Langford,
  Luo, and Schapire]{agarwal2017open}
Alekh Agarwal, Akshay Krishnamurthy, John Langford, Haipeng Luo, and Robert~E.
  Schapire.
\newblock Open problem: First-order regret bounds for contextual bandits.
\newblock In \emph{Conference on Learning Theory}, 2017{\natexlab{a}}.

\bibitem[Agarwal et~al.(2017{\natexlab{b}})Agarwal, Luo, Neyshabur, and
  Schapire]{agarwal2017corralling}
Alekh Agarwal, Haipeng Luo, Behnam Neyshabur, and Robert~E Schapire.
\newblock Corralling a band of bandit algorithms.
\newblock \emph{Conference on Learning Theory}, 2017{\natexlab{b}}.

\bibitem[Allen-Zhu et~al.(2018)Allen-Zhu, Bubeck, and Li]{allen2018make}
Zeyuan Allen-Zhu, S{\'e}bastien Bubeck, and Yuanzhi Li.
\newblock Make the minority great again: First-order regret bound for
  contextual bandits.
\newblock \emph{International Conference on Machine Learning}, 2018.

\bibitem[Allenberg et~al.(2006)Allenberg, Auer, Gy{\"o}rfi, and
  Ottucs{\'a}k]{allenberg2006hannan}
Chamy Allenberg, Peter Auer, L{\'a}szl{\'o} Gy{\"o}rfi, and Gy{\"o}rgy
  Ottucs{\'a}k.
\newblock Hannan consistency in on-line learning in case of unbounded losses
  under partial monitoring.
\newblock In \emph{International Conference on Algorithmic Learning Theory}.
  Springer, 2006.

\bibitem[Auer et~al.(2002)Auer, Cesa-Bianchi, Freund, and
  Schapire]{auer2002nonstochastic}
Peter Auer, Nicolo Cesa-Bianchi, Yoav Freund, and Robert~E. Schapire.
\newblock The nonstochastic multiarmed bandit problem.
\newblock \emph{SIAM Journal on Computing}, 2002.

\bibitem[Auer et~al.(2018)Auer, Gajane, and Ortner]{auer2018adaptively}
Peter Auer, Pratik Gajane, and Ronald Ortner.
\newblock Adaptively tracking the best arm with an unknown number of
  distribution changes.
\newblock In \emph{European Workshop on Reinforcement Learning}, 2018.

\bibitem[Bartlett et~al.(2002)Bartlett, Boucheron, and
  Lugosi]{bartlett2002model}
Peter~L Bartlett, St{\'e}phane Boucheron, and G{\'a}bor Lugosi.
\newblock Model selection and error estimation.
\newblock \emph{Machine Learning}, 2002.

\bibitem[Bastani et~al.(2017)Bastani, Bayati, and Khosravi]{bastani2017mostly}
Hamsa Bastani, Mohsen Bayati, and Khashayar Khosravi.
\newblock Mostly exploration-free algorithms for contextual bandits.
\newblock \emph{arXiv:1704.09011}, 2017.

\bibitem[Ben-David et~al.(1995)Ben-David, Cesa-Bianchi, Haussler, and
  Long]{bendavid1995characterizations}
Shai Ben-David, Nicolo Cesa-Bianchi, David Haussler, and Philip~M Long.
\newblock Characterizations of learnability for classes of (0,..., n)-valued
  functions.
\newblock \emph{Journal of Computer and System Sciences}, 1995.

\bibitem[Beygelzimer et~al.(2011)Beygelzimer, Langford, Li, Reyzin, and
  Schapire]{beygelzimer2011contextual}
Alina Beygelzimer, John Langford, Lihong Li, Lev Reyzin, and Robert Schapire.
\newblock Contextual bandit algorithms with supervised learning guarantees.
\newblock In \emph{International Conference on Artificial Intelligence and
  Statistics}, 2011.

\bibitem[Birg{\'e} and Massart(1998)]{birge1998minimum}
Lucien Birg{\'e} and Pascal Massart.
\newblock Minimum contrast estimators on sieves: exponential bounds and rates
  of convergence.
\newblock \emph{Bernoulli}, 1998.

\bibitem[Brown et~al.(2007)Brown, Levine, et~al.]{brown2007variance}
Lawrence~D Brown, Michael Levine, et~al.
\newblock Variance estimation in nonparametric regression via the difference
  sequence method.
\newblock \emph{The Annals of Statistics}, 2007.

\bibitem[Chatterji et~al.(2019)Chatterji, Muthukumar, and
  Bartlett]{chatterji2019osom}
Niladri~S. Chatterji, Vidya Muthukumar, and Peter~L. Bartlett.
\newblock Osom: A simultaneously optimal algorithm for multi-armed and linear
  contextual bandits.
\newblock \emph{arXiv:1905.10040}, 2019.

\bibitem[Chen et~al.(2019)Chen, Lee, Luo, and Wei]{chen2019new}
Yifang Chen, Chung-Wei Lee, Haipeng Luo, and Chen-Yu Wei.
\newblock A new algorithm for non-stationary contextual bandits: Efficient,
  optimal, and parameter-free.
\newblock \emph{Conference on Learning Theory}, 2019.

\bibitem[Cheung et~al.(2019)Cheung, Simchi-Levi, and Zhu]{cheung2018learning}
Wang~Chi Cheung, David Simchi-Levi, and Ruihao Zhu.
\newblock Learning to optimize under non-stationarity.
\newblock \emph{International Conference on Artificial Intelligence and
  Statistics}, 2019.

\bibitem[Chu et~al.(2011)Chu, Li, Reyzin, and Schapire]{chu2011contextual}
Wei Chu, Lihong Li, Lev Reyzin, and Robert Schapire.
\newblock Contextual bandits with linear payoff functions.
\newblock In \emph{International Conference on Artificial Intelligence and
  Statistics}, 2011.

\bibitem[Cutkosky and Boahen(2017)]{cutkosky2017online}
Ashok Cutkosky and Kwabena~A. Boahen.
\newblock Online learning without prior information.
\newblock \emph{Conference on Learning Theory}, 2017.

\bibitem[Daniely et~al.(2015)Daniely, Sabato, Ben-David, and
  Shalev-Shwartz]{daniely2015multiclass}
Amit Daniely, Sivan Sabato, Shai Ben-David, and Shai Shalev-Shwartz.
\newblock Multiclass learnability and the {ERM} principle.
\newblock \emph{The Journal of Machine Learning Research}, 2015.

\bibitem[de~la Pe{\~n}a and Gin\'e(1998)]{delapena1998decoupling}
Victor de~la Pe{\~n}a and Evarist Gin\'e.
\newblock \emph{Decoupling: From Dependence to Independence}.
\newblock Springer, 1998.

\bibitem[Devroye et~al.(1996)Devroye, Gy\"{o}rfi, and Lugosi]{DevroyeLugosi96}
Luc Devroye, L\'{a}zl\'{o} Gy\"{o}rfi, and G\'{a}bor Lugosi.
\newblock \emph{A Probabilistic Theory of Pattern Recognition}.
\newblock Springer, 1996.

\bibitem[Dicker(2014)]{dicker2014variance}
Lee~H Dicker.
\newblock Variance estimation in high-dimensional linear models.
\newblock \emph{Biometrika}, 2014.

\bibitem[Foster et~al.(2017)Foster, Kale, Mohri, and
  Sridharan]{foster2017parameter}
Dylan~J. Foster, Satyen Kale, Mehryar Mohri, and Karthik Sridharan.
\newblock Parameter-free online learning via model selection.
\newblock In \emph{Advances in Neural Information Processing Systems}, 2017.

\bibitem[Foster et~al.(2018)Foster, Agarwal, Dudik, Luo, and
  Schapire]{foster2018practical}
Dylan~J. Foster, Alekh Agarwal, Miroslav Dudik, Haipeng Luo, and Robert
  Schapire.
\newblock Practical contextual bandits with regression oracles.
\newblock In \emph{International Conference on Machine Learning}, 2018.

\bibitem[Freund and Schapire(1997)]{freund1997decision}
Yoav Freund and Robert~E Schapire.
\newblock A decision-theoretic generalization of on-line learning and an
  application to boosting.
\newblock \emph{Journal of Computer and System Sciences}, 1997.

\bibitem[Gaillard et~al.(2014)Gaillard, Stoltz, and
  Van~Erven]{gaillard2014second}
Pierre Gaillard, Gilles Stoltz, and Tim Van~Erven.
\newblock A second-order bound with excess losses.
\newblock In \emph{Conference on Learning Theory}, 2014.

\bibitem[Haussler and Long(1995)]{haussler1995generalization}
David Haussler and Philip~M Long.
\newblock A generalization of {S}auer's lemma.
\newblock \emph{Journal of Combinatorial Theory, Series A}, 1995.

\bibitem[Kannan et~al.(2018)Kannan, Morgenstern, Roth, Waggoner, and
  Wu]{kannan2018smoothed}
Sampath Kannan, Jamie~H Morgenstern, Aaron Roth, Bo~Waggoner, and Zhiwei~Steven
  Wu.
\newblock A smoothed analysis of the greedy algorithm for the linear contextual
  bandit problem.
\newblock In \emph{Advances in Neural Information Processing Systems}, 2018.

\bibitem[Koltchinskii(2001)]{Koltchinskii2001}
Vladimir Koltchinskii.
\newblock Rademacher penalties and structural risk minimization.
\newblock \emph{{IEEE} Transactions on Information Theory}, 2001.

\bibitem[Kong and Valiant(2018)]{kong2018estimating}
Weihao Kong and Gregory Valiant.
\newblock Estimating learnability in the sublinear data regime.
\newblock In \emph{Advances in Neural Information Processing Systems}, 2018.

\bibitem[Koolen and Van~Erven(2015)]{koolen2015second}
Wouter~M Koolen and Tim Van~Erven.
\newblock Second-order quantile methods for experts and combinatorial games.
\newblock In \emph{Conference on Learning Theory}, 2015.

\bibitem[Krishnamurthy et~al.(2016)Krishnamurthy, Agarwal, and
  Dudik]{krishnamurthy2016contextual}
Akshay Krishnamurthy, Alekh Agarwal, and Miro Dudik.
\newblock Contextual semibandits via supervised learning oracles.
\newblock In \emph{Advances In Neural Information Processing Systems}, 2016.

\bibitem[Krishnamurthy et~al.(2018)Krishnamurthy, Wu, and
  Syrgkanis]{krishnamurthy2018semiparametric}
Akshay Krishnamurthy, Zhiwei~Steven Wu, and Vasilis Syrgkanis.
\newblock Semiparametric contextual bandits.
\newblock In \emph{International Conference on Machine Learning}, 2018.

\bibitem[Krishnamurthy et~al.(2019)Krishnamurthy, Langford, Slivkins, and
  Zhang]{krishnamurthy2019contextual}
Akshay Krishnamurthy, John Langford, Aleksandrs Slivkins, and Chicheng Zhang.
\newblock Contextual bandits with continuous actions: Smoothing, zooming, and
  adapting.
\newblock \emph{Conference on Learning Theory}, 2019.

\bibitem[Langford and Zhang(2008)]{langford2008epoch}
John Langford and Tong Zhang.
\newblock The epoch-greedy algorithm for multi-armed bandits with side
  information.
\newblock In \emph{Advances in neural information processing systems}, 2008.

\bibitem[Lattimore(2015)]{lattimore2015pareto}
Tor Lattimore.
\newblock The pareto regret frontier for bandits.
\newblock In \emph{Advances in Neural Information Processing Systems}, 2015.

\bibitem[Li et~al.(2017)Li, Lu, and Zhou]{li2017provably}
Lihong Li, Yu~Lu, and Dengyong Zhou.
\newblock Provably optimal algorithms for generalized linear contextual
  bandits.
\newblock In \emph{International Conference on Machine Learning}, 2017.

\bibitem[Locatelli and Carpentier(2018)]{locatelli2018adaptivity}
Andrea Locatelli and Alexandra Carpentier.
\newblock Adaptivity to smoothness in x-armed bandits.
\newblock In \emph{Conference on Learning Theory}, 2018.

\bibitem[Lugosi and Nobel(1999)]{lugosi1999adaptive}
G{\'a}bor Lugosi and Andrew~B Nobel.
\newblock Adaptive model selection using empirical complexities.
\newblock \emph{Annals of Statistics}, 1999.

\bibitem[Luo and Schapire(2015)]{luo2015achieving}
Haipeng Luo and Robert~E Schapire.
\newblock Achieving all with no parameters: Adanormalhedge.
\newblock In \emph{Conference on Learning Theory}, 2015.

\bibitem[Luo et~al.(2018)Luo, Wei, Agarwal, and Langford]{luo2017efficient}
Haipeng Luo, Chen-Yu Wei, Alekh Agarwal, and John Langford.
\newblock Efficient contextual bandits in non-stationary worlds.
\newblock \emph{Conference on Learning Theory}, 2018.

\bibitem[Lykouris et~al.(2018)Lykouris, Sridharan, and
  Tardos]{lykouris2017small}
Thodoris Lykouris, Karthik Sridharan, and {\'E}va Tardos.
\newblock Small-loss bounds for online learning with partial information.
\newblock \emph{Conference on Learning Theory}, 2018.

\bibitem[Massart(2007)]{massart2007concentration}
Pascal Massart.
\newblock \emph{Concentration inequalities and model selection}.
\newblock Springer, 2007.

\bibitem[McMahan and Abernethy(2013)]{mcmahan2013minimax}
Brendan McMahan and Jacob Abernethy.
\newblock Minimax optimal algorithms for unconstrained linear optimization.
\newblock In \emph{Advances in Neural Information Processing Systems}, 2013.

\bibitem[Neu(2015)]{neu2015explore}
Gergely Neu.
\newblock Explore no more: Improved high-probability regret bounds for
  non-stochastic bandits.
\newblock In \emph{Advances in Neural Information Processing Systems}, 2015.

\bibitem[Orabona(2014)]{orabona2014simultaneous}
Francesco Orabona.
\newblock Simultaneous model selection and optimization through parameter-free
  stochastic learning.
\newblock In \emph{Advances in Neural Information Processing Systems}, 2014.

\bibitem[Raghavan et~al.(2018)Raghavan, Slivkins, Vaughan, and
  Wu]{raghavan2018externalities}
Manish Raghavan, Aleksandrs Slivkins, Jennifer~Wortman Vaughan, and
  Zhiwei~Steven Wu.
\newblock The externalities of exploration and how data diversity helps
  exploitation.
\newblock \emph{Conference on Learning Theory}, 2018.

\bibitem[Russo and Van~Roy(2013)]{russo2013eluder}
Daniel Russo and Benjamin Van~Roy.
\newblock Eluder dimension and the sample complexity of optimistic exploration.
\newblock In \emph{Advances in Neural Information Processing Systems}, 2013.

\bibitem[Shawe-Taylor et~al.(1998)Shawe-Taylor, Bartlett, Williamson, and
  Anthony]{shawe1998structural}
John Shawe-Taylor, Peter~L. Bartlett, Robert~C Williamson, and Martin Anthony.
\newblock Structural risk minimization over data-dependent hierarchies.
\newblock \emph{IEEE Transactions on Information Theory}, 1998.

\bibitem[Vapnik(1982)]{vapnik1982estimation}
Vladimir Vapnik.
\newblock \emph{Estimation of dependences based on empirical data}.
\newblock Springer-Verlag, 1982.

\bibitem[Vapnik(1992)]{vapnik1992principles}
Vladimir Vapnik.
\newblock Principles of risk minimization for learning theory.
\newblock In \emph{Advances in Neural Information Processing Systems}, 1992.

\bibitem[Vershynin(2012)]{vershynin2012introduction}
Roman Vershynin.
\newblock \emph{Introduction to the non-asymptotic analysis of random
  matrices}.
\newblock Cambridge University Press, 2012.

\bibitem[Verzelen and Gassiat(2018)]{verzelen2018adaptive}
Nicolas Verzelen and Elisabeth Gassiat.
\newblock Adaptive estimation of high-dimensional signal-to-noise ratios.
\newblock \emph{Bernoulli}, 2018.

\bibitem[Wang et~al.(2008)Wang, Brown, Cai, Levine, et~al.]{wang2008effect}
Lie Wang, Lawrence~D Brown, T~Tony Cai, Michael Levine, et~al.
\newblock Effect of mean on variance function estimation in nonparametric
  regression.
\newblock \emph{The Annals of Statistics}, 2008.

\end{thebibliography}

\vfill
\newpage
\appendix
\section{Omitted details for~\pref{sec:preliminaries}}
\label{app:calculations}

In this section we provide more detail as to why various natural approaches do not provide a
satisfactory resolution to the model selection problem for contextual
bandits.
We consider a more general setup here, where the learner is given a set of policy classes $\Pi_1, \ldots, \Pi_M$, each of which contains a set of arbitrary mappings from the context space $\Xcal$ to $\cA$.
The (expected) regret to class $m$ is 
\[
\Reg(\Pi_m) \defeq \sum_{t=1}^T\EE\sbr{\ell_t(a_t)} - \min_{\pi \in \Pi_m} \sum_{t=1}^T\EE\sbr{\ell_t(\pi(x_t))}.
\]
Let $\ms \defeq \argmin_m \min_{\pi \in \Pi_m} \sum_{t=1}^T\EE\sbr{\ell_t(\pi(x_t))}$ be the index of the class containing the optimal policy.
The goal is to achieve $\Reg(\Pi_{\ms}) = O\rbr{T^{\alpha}
  \cdot\textrm{comp}^{1-\alpha}(\Pi_{\ms})}$ for some $\alpha \in
[\nicefrac{1}{2},1)$ without knowing $\ms$ ahead of time (ignoring the
dependence on $K$). Our realizable linear setting is clearly a special case of this setup.
It is well-known that the model selection guarantee we desire can be
achieved in the full-information online learning setting,
and below we discuss the difficulties of extending these approaches to the bandit setting, even when $M=2$.
For simplicity we consider finite classes, so the complexity of $\Pi_m$ is measured by $\log|\Pi_m|$.

\subsection{Running \textsf{Hedge} with all policy classes}
The classical contextual bandit algorithm \textsf{Exp4}~\citep{auer2002nonstochastic} is based on the full-information algorithm \textsf{Hedge}~\citep{freund1997decision}. Fix a policy class $\Pi$.
At each time $t$, \textsf{Exp4} computes a distribution $p_t$ over the policies in $\Pi$ using the exponential weight update rule:
\[
p_t(\pi) \propto p_0(\pi)\exp\left(\eta\sum_{\tau<t} \hat{\ell}_\tau(\pi(x_\tau)) \right),
\]
where $\eta$ is a learning rate parameter, $p_0$ is a prespecified prior distribution over the policies, and $\hat{\ell}_\tau$ is the importance-weighted loss estimator, defined as $\hat{\ell}_\tau(a) \defeq \frac{\ell_\tau(a)\II\{a_\tau = a\}}{\sum_{\pi\in\Pi: \pi(x_\tau)=a} p_\tau(\pi)}$ for all $a$.
\textsf{Exp4} ensures for any $\pi^\star \in \Pi$,
\begin{align}
\sum_{t=1}^T\EE\sbr{\ell_t(a_t)} - \sum_{t=1}^T\EE\sbr{\ell_t(\pi^\star(x_t))}
&\leq \frac{\log\rbr{\frac{1}{p_0(\pi^\star)}}}{\eta} + \eta\EE\sbr{\sum_{t=1}^T \sum_{\pi\in\Pi} p_t(\pi)\hat{\ell}_t(\pi(x_t))^2} \label{eq:hedge_local_norm}\\
&\leq \frac{\log\rbr{\frac{1}{p_0(\pi^\star)}}}{\eta} + \eta TK. \notag
\end{align}
Now, consider running this algorithm with $\Pi \defeq \bigcup_{m=1}^M \Pi_m$. 
With a uniform prior and the optimal tuning of $\eta$, this leads the following regret bound for all $m$, which is clearly undesirable:
\[
\Reg(\Pi_m) = O\rbr{\sqrt{TK\log\rbr{\sum_{m=1}^M|\Pi_m|}}}.
\]
On the other hand, if we set $p_0(\pi) \defeq 1/(M|\Pi_m|)$, where $m$ is the least index such that $\pi \in \Pi_m$, then we have that for each $m$,
\[
\Reg(\Pi_m) \leq \frac{\log\rbr{M |\Pi_m|}}{\eta} + \eta TK.
\]
With oracle tuning for $\eta$ this would give $\Reg(\Pi_{m^\star}) = \Reg \leq O(\sqrt{TK \log |\Pi_{m^\star}|})$, as desired. 
The issue, of course, is that tuning requires knowing $\ms$ ahead of time.
One can instead simply set $\eta = 1/\sqrt{TK}$ and obtain for all $m$,
\[
\Reg(\Pi_m) = O\rbr{\sqrt{TK}\log\rbr{M|\Pi_m|}},
\]
which is not a satisfactory model selection guarantee. In particular
this bound is vacuous when $\log |\Pi_{m^\star}| = \Omega(\sqrt{T})$,
even though we could have achieved sublinear regret here if $\ms$ were known.

A natural next step would be to apply an individual learning rate $\eta_m$ for each class separately, with the hope of achieving for all $m$,
\[
\Reg(\Pi_m) = \frac{\log\rbr{M |\Pi_m|}}{\eta_m} + \eta_m TK,
\]
which will then solve the problem by setting $\eta_m = \sqrt{\log(|\Pi_m|)/(TK)}$.
However, we are not aware of any existing approaches that achieve this guarantee.
The closest guarantee (achieved by variants of \textsf{Hedge}~\citep{gaillard2014second, koolen2015second}) is that
for all $m$ and $\pi_m \in \Pi_m$, 
\begin{align*}
\sum_{t=1}^T\EE\sbr{\ell_t(a_t)} - \sum_{t=1}^T\EE\sbr{\ell_t(\pi_m(x_t))}
&\leq \frac{\log\rbr{M|\Pi_m|}}{\eta_m} + \eta_m \EE\sbr{\sum_{t=1}^T (\ell_t(a_t) - \hat{\ell}_t(\pi_m(x_t)))^2} \\
&\leq \frac{\log\rbr{M|\Pi_m|}}{\eta_m} + 2\eta_m T + 2\eta_m \EE\sbr{\sum_{t=1}^T \hat{\ell}_t(\pi_m(x_t))^2}.
\end{align*}
Unfortunately, this does not enjoy the useful local norm term as in~\pref{eq:hedge_local_norm}, and in particular the last term involving the second moment of the loss estimator can be unbounded.
Common fixes such as forcing a small amount of uniform exploration all lead to further tuning issues.

\subsection{Aggregating via \textsf{Corral}}

Next, we briefly describe issues with using \textsf{Corral}~\citep{agarwal2017corralling} to aggregate multiple instances of \textsf{Exp4}.
In short, using Theorem 4 of~\citep{agarwal2017corralling}, one can verify that the algorithm ensures for all $m$,
\[
\Reg(\Pi_m) = \otil\rbr{\frac{M}{\eta} + T\eta - \frac{\rho_m}{\eta} + \sqrt{TK\log(|\Pi_m|)\rho_m}},
\]
for a fixed parameter $\eta$ and a certain data-dependent quantity $\rho_m$.
Using the AM-GM inequality to upper bound the last term by $\eta TK\log(|\Pi_m|) + \rho_m/\eta$, and canceling with the third term, we have $\Reg(\Pi_m) = \otil\rbr{\frac{M}{\eta}+\eta TK\log(|\Pi_m|)}$.
At this point, one can see that the tuning issues similar to those discussed in the previous section appear, and again there is no obvious fix.

\subsection{Adaptive $\epsilon$-greedy}

Here we present a natural adaptation of an $\epsilon$-greedy algorithm
for model selection with two classes $|\Pi_1| \ll |\Pi_2|$. The
algorithm does not achieve a satisfactory model selection guarantee,
but we include the analysis because it demonstrates some of the
difficulties with parameter tuning, even for algorithms based on naive exploration 
where the best rate one could hope to achieve is $O(T^{2/3}\mathrm{comp}(\cF_{\ms})^{1/3})$.

\begin{algorithm}
\begin{algorithmic}
\STATE Inputs: $T$, policy classes $\Pi_1, \Pi_2$ with $|\Pi_1| \leq
|\Pi_2|$.  
\STATE Set $t_1 \defeq \lceil T^{2/3}(K\log |\Pi_1|)^{1/3}\rceil$, $t_2 \defeq \lceil T^{2/3}(K\log |\Pi_2|)^{1/3}\rceil$. 
\STATE Set $\Delta \defeq \Omega\rbr{\rbr{\frac{K}{T \log (T|\Pi_1)}}^{1/3}\sqrt{\log(T|\Pi_2|)}}$.
\FOR{$t=1,\ldots,t_1$} 
\STATE Observe $x_t$ choose $a_t \sim \textrm{Unif}(\Acal)$, observe $\ell_t(a_t)$. 
\ENDFOR 
\STATE Define $\hat{L}_1: \pi \mapsto \sum_{t=1}^{t_1} \ell_t(a_t)\one\{\pi(x_t)=a_t\}$. Set
\begin{align*}
\hat{\pi}_1 \defeq \argmin_{\pi \in \Pi_1} \hat{L}_1(\pi), \qquad 
\hat{\pi}_2 \defeq \argmin_{\pi \in \Pi_2} \hat{L}_1(\pi)
\end{align*}
\IF{$\hat{L}_1(\hat{\pi}_1) \leq \hat{L}_1(\hat{\pi}_2) + \Delta$}
\STATE Use $\hat{\pi}_1$ to select actions for the rest of time.
\ELSE
\STATE Explore uniformly for a total of $t_2$ rounds, let $\hat{\pi}^{(2)}_2 \defeq \argmin_{\pi \in \Pi_2} \sum_{t=1}^{t_2}\ell_t(a_t)\one\{\pi(x_t)=a_t\}$, and use $\hat{\pi}_2^{(2)}$ to select actions for the rest of time.
\ENDIF
\end{algorithmic}
\caption{An adaptive explore-first algorithm.}
\label{alg:eps_greedy}
\end{algorithm}

Pseudocode is displayed in~\pref{alg:eps_greedy}. The algorithm
operates in the stochastic setting, where we have
$(x_t,\ell_t)\sim\Dcal$ on each round for some distribution $\Dcal$.
Define $L(\pi) \defeq \EE_{(x,\ell)\sim \Dcal}\sbr{\ell(\pi(x))}$, and let
$\pi_i^\star \defeq \argmin_{\pi \in \Pi_i} L(\pi)$.  
We assume that losses are bounded in $[0,1]$.

The algorithm consists of an
exploration phase, a statistical test, a second exploration
phase depending on the outcome of the test, and then an exploitation
phase. The intuition is that we first explore for $t_1$ rounds, where
$t_1$ is the optimal hyperparameter choice for the smaller class
$\Pi_1$ (smaller classes require less exploration). Then, we perform a
statistical test to determine if $\Pi_1$ can achieve loss that is
competitive with $\Pi_2$. If this is the case, we simply exploit $\Pi_1$
for the remaining rounds. Otherwise, we continue exploring for a
total of $t_2$ rounds, where $t_2$ is the optimal hyperparameter for
$\Pi_2$. We finish by exploiting with the empirical risk minimizer for
$\Pi_2$.

\begin{proposition}
\label{prop:epsgreedy}
In the stochastic setting, \pref{alg:eps_greedy} achieves the following guarantees simultaneously with probability at least $1-1/T$:
\begin{align*}
\Reg(\Pi_1) \leq \otil\rbr{T^{2/3}(K\log|\Pi_1|)^{1/3}}, \quad\text{and}\quad \Reg(\Pi_2) \leq \otil\rbr{T^{2/3}K^{1/3} \frac{\log^{1/2}|\Pi_2|}{\log^{1/3}|\Pi_1|} }.
\end{align*}
\end{proposition}
Note that this is \emph{not} a satisfactory model selection guarantee,
since the exponents on $T$ and the policy complexity for $\Pi_2$ do
not sum to $1$ as we would like. Conceptually, if the algorithm
exploits with a fixed policy, it must first determine whether $\Pi_2$
offers much lower loss than $\Pi_1$ so that it can decide which class
to use for exploitation. To make this determination, it must estimate
the optimal loss for both classes. Unfortunately, with too little
exploration data, we will significantly underestimate the loss for
$\Pi_2$, and with too much data we will compromise the regret bound
for $\Pi_1$. We are not aware of an approach to balance these
competing objectives with this style of algorithm.
\begin{proof}[\pfref{prop:epsgreedy}]
Define $\hat{L}_2: \pi \mapsto \sum_{t=1}^{t_2}
\ell_t(a_t)\one\{\pi(x_t)=a_t\}$. We will only use $\hat{L}_2$ in the
event that we continue to explore after the test.  Via Bernstein's
inequality (and using that the deviation is at most $1$), the
following inequalities all hold with probability at least $1-\delta$:
\begin{align*}
\forall j \in [2], \forall \pi \in \Pi_i: \abr{L(\pi) - \hat{L}_j(\pi)} &\leq 
c \sqrt{\frac{K\log(|\Pi_i|/\delta)}{t_j}},\\
\forall i \in [2]: \abr{L(\pi_i^\star) - \hat{L}_1(\pi_i^\star)} &\leq 
c\sqrt{\frac{K \log(1/\delta)}{t_1}}.
\end{align*}
Note that the second inequality does not use uniform convergence, and it applies only to each $\pi_i^\star$.
Appealing to the standard explore-first analysis, we know that the
regret bound for $\Pi_1$ is achieved if we exploit with $\hat{\pi}_1$,
and similarly for $\Pi_2$ if we exploit with $\hat{\pi}_2^{(2)}$. We
are left to verify the other two cases.  Let us consider $\Reg(\Pi_2)$
when we only explore for $t_1$ rounds. We have
\begin{align*}
\Reg(\Pi_2) &= t_1 + (T-t_1)\rbr{L(\hat{\pi}_1) - L(\pi_2^\star)}\\
& \leq t_1 + T\rbr{c \sqrt{\frac{K \log (|\Pi_1|/\delta)}{t_1}}  + \Delta + c\sqrt{\frac{K\log(1/\delta)}{t_1}}}\\
& \leq t_1 + 2c\sqrt{\frac{KT\log (|\Pi_1|/\delta)}{t_1}} + T\Delta.
\end{align*}
Here we use the two concentration inequalities above, the fact that
$\hat{L}_1(\pi_2^\star) \geq \hat{L}_1(\hat{\pi}_2)$, and the fact that
the test succeeded. Note that we did not require uniform convergence
over $\Pi_2$ for this argument. Now, let us consider the regret for $\Pi_1$ when we explore for $t_2$ rounds.
\begin{align*}
& \Reg(\Pi_1) \\&=  t_2 + (T-t_2) \rbr{L(\hat{\pi}_2^{(2)}) - L(\pi_1^\star)}\\
& \leq  t_2 + (T-t_2)\rbr{c\sqrt{\frac{K \log (|\Pi_2|/\delta) }{t_2}} + L(\pi_2^\star) - \hat{L}_1(\hat{\pi}_1) + c\sqrt{\frac{K \log(1/\delta)}{t_1}}}\\
& \leq  t_2 + (T-t_2)\rbr{c\sqrt{\frac{K \log (|\Pi_2|/\delta) }{t_2}} + c\sqrt{\frac{K \log(|\Pi_2|/\delta)}{t_1}} + \hat{L}_1(\hat{\pi}_2) - \hat{L}_1(\hat{\pi}_1) + c\sqrt{\frac{K \log(1/\delta)}{t_1}}}\\
& \leq 4c(T-t_2)\sqrt{\frac{K\log(|\Pi_2|/\delta)}{t_1}} - (T-t_2)\Delta,
\end{align*}
where we have used that $t_2$ is lower order than the deviation bound
term, since $t_1$ is much smaller. To obtain the claimed regret bound for $\Pi_1$, we must choose $\Delta$ as
\begin{align*}
\Delta \geq 4c\sqrt{\frac{K \log (|\Pi_2|/\delta)}{t_1}} - \Omega\rbr{(K\log(|\Pi_1|/\delta)/T)^{1/3}} = \Omega\rbr{\rbr{\frac{K}{T\log (|\Pi_1|/\delta)}}^{1/3}\sqrt{\log (|\Pi_2|/\delta)}},
\end{align*}
with $\delta=1/T$.
\end{proof}

\section{Preliminaries for main results}
This section consists of self-contained technical results used to
prove the main theorem.
\subsection{Properties of subgaussian and subexponential random
  variables}
\label{app:subgaussian}
Here we state some standard facts about subgaussian and subexponential
random variables that will be used throughout the analysis. The reader
may consult e.g., \cite{vershynin2012introduction}, for proofs.

Note that if $z\sim\subG_d(\tau^{2})$ then we clearly have
$\tri*{z,\theta}\sim\subG(\tau^{2}\nrm*{\theta}_{2}^{2})$ and likewise
if $z\sim\subE_d(\lambda)$ then
$\tri*{z,\theta}\sim\subE(\lambda\nrm*{\theta}_{2})$. Furthermore, if
$z\sim\subg_d(\sigma^2)$, then $\En\brk*{zz^{\trn}}\preceq \sigma^{2}\cdot{}I_{d}$.

\begin{proposition}
  \label{prop:subgaussian_equivalence}
There exist universal constants $c_1,c_2,c_3>0$ such that for any random variable $z\sim\subG(\sigma^{2})$, the following hold:
\begin{itemize}
\item
  $\Pr(\abs*{z}>t)\leq{}2e^{-\frac{t^{2}}{c_1\sigma^{2}}}\quad\forall{}t\geq{}0$.
\item $\En\brk{\exp\prn{c_2\abs*{z}^{2}/\sigma^{2}}}\leq{}2$.
\item If $\En\brk*{z}=0$, $\En\brk*{\exp\prn*{sz}}\leq{}\exp\prn*{c_3s^{2}\sigma^{2}}\quad\forall{}s\in\bbR^{d}$.
\end{itemize}
Moreover, there is some universal constant $c_4$ such that if any of
the above properties hold, then
$z\sim\subG(c_4\sigma^{2})$
\end{proposition}

\begin{proposition}
  \label{prop:subexponential_equivalence}
There exist universal constants $c_1,c_2,c_3>0$ such that for any random variable $z\sim\subE(\lambda)$, the following hold:
\begin{itemize}
\item
  $\Pr(\abs*{z}>t)\leq{}2e^{-\frac{t}{c_1\lambda}}\quad\forall{}t\geq{}0$.
\item If $\En\brk*{z}=0$, $\En\brk*{\exp\prn*{sz}}\leq{}\exp\prn*{c_2s^{2}\lambda^{2}}\quad\forall{}s:\abs*{s}\leq{}\frac{1}{c_3\lambda}$.
\end{itemize}
Moreover, there is some universal constant $c_4$ such that if any of
the above properties hold, then
$z\sim\subE(c_4\lambda)$
\end{proposition}
\begin{proposition}
  \label{prop:subgaussian_closure}
There is a universal constant $c>0$ such that the following hold:
\begin{itemize}
\item If $z\sim\subG(\sigma^2)$, $z-\En\brk*{z}\sim\subG(4\sigma^2)$,
  and if $z\sim\subE(\lambda)$, $z-\En\brk*{z}\sim\subE(2\lambda)$.
\item If $z\sim\subg(\sigma^2)$, then
  $z^{2}-\En\brk*{z^{2}}\sim\sube(c\cdot\sigma^{2})$.
\item If $z_1\sim\subg(\sigma^2_1)$ and $z_2\sim\subg(\sigma_2^2)$, then $z_1z_2-\En\brk*{z_1z_2}\sim{}\sube(c\cdot\sigma_1\sigma_2)$.
\end{itemize}
\end{proposition}

\begin{proposition}[Bernstein's inequality for subexponential random variables]
\label{prop:bernstein_subexponential}
Let $z_1,\ldots,z_n$ be independent mean-zero random variables with
$z_i\sim\subE(\lambda)$ for all $i$, and let
$Z=\sum_{i=1}^{n}z_i$. Then for some universal constant $c>0$,
\[
\Pr\prn*{\abs*{Z}\geq{}t}\leq{}2\exp\prn*{
-c\prn*{\frac{t^{2}}{\lambda^2n}\wedge\frac{t}{\lambda}}
}.
\]
 In particular, with probability at least $1-\delta$, $\abs*{Z} \leq{} O\prn*{\sqrt{n\lambda^{2}\log(2/\delta)} + \lambda\log(2/\delta)}$.
\end{proposition}

\begin{lemma}
\label{lem:subexponential_max}
Let $z_{1},\ldots{}z_{n}$ be i.i.d. draws of a mean-zero random variable
$z\sim\subE_{d}(\lambda)$, and let $Z_{i}=\sum_{t=1}^{i}z_t$. Then
with probability at least $1-\delta$,
\[
\max_{i\in\brk*{n}}\nrm*{Z_i}_{2}
\leq{} O\prn*{\lambda\sqrt{dn\log(2d/\delta)} + \lambda{}d^{1/2}\log(2d/\delta)}.
\]
\end{lemma}
\begin{proof}[\pfref{lem:subexponential_max}]
  We first claim that for any $s>0$, the sequence $X_i\ldef{}e^{s\nrm*{Z_{i}}_{2}}$
  is a non-negative submartingale. Indeed, using Jensen's inequality,
  we have
\[
\En\brk*{X_i\mid{}z_1,\ldots,z_{i-1}} =
\En\brk*{e^{s\nrm*{\sum_{t=1}^{i}z_t}_{2}}\mid{}z_1,\ldots,z_{i-1}}
\geq{} \En\brk*{e^{s\nrm*{\sum_{t=1}^{i-1}z_t}_{2}}\mid{}z_1,\ldots,z_{i-1}}=X_{i-1}.
\]
Thus, applying the Chernoff method along with Doob's
  maximal inequality, we have
  \[
  \Pr\prn*{\max_{i\in\brk*{n}}\nrm*{Z_i}_{2}>t} \leq{}
  \En\brk*{\exp\prn*{s\nrm*{\sum_{t=1}^{n}z_t}_{2}-st}}.
  \]
  We apply the bound
  \[
  \En\brk*{\exp\prn*{s\nrm*{\sum_{t=1}^{n}z_t}_{2}}}
  \leq{}
  \En\brk*{\exp\prn*{s\sqrt{d}\nrm*{\sum_{t=1}^{n}z_t}_{\infty}}}
  \leq{}
  \sum_{k=1}^{d}\En\brk*{\exp\prn*{s\sqrt{d}\abs*{\tri*{\sum_{t=1}^{n}z_t,e_k}}}},
  \]
where $e_k$ is the $k$th standard basis vector.  Since $z_t\sim\subE_d(\lambda)$ and $\crl*{z_t}$ are
  independent, the latter quantity is bounded by
  \[
  2d\cdot\exp\prn*{Cs^{2}d\lambda^{2}n},
  \]
  so long as $s\leq{}1/c\lambda{}\sqrt{d}$, for absolute constants
  $C,c>0$. We set
  $s\propto\frac{t}{d\lambda^{2}n}\wedge\frac{1}{\lambda{}\sqrt{d}}$
  to conclude
  \[
  \Pr\prn*{\max_{i\in\brk*{n}}\nrm*{Z_i}_{2}>t}\leq{}
  2d\cdot{}e^{-C\prn*{\frac{t^{2}}{d\lambda^{2}n}\vee\frac{t}{\lambda{}\sqrt{d}}}}.
  \]
  Or in other words, with probability at least $1-\delta$,
\begin{align*}
\max_{i\in\brk*{n}}\nrm*{Z_i}_{2}
\leq{} O\prn*{\lambda\sqrt{dn\log(2d/\delta)} + \lambda{}d^{1/2}\log(2d/\delta)}.\tag*\qedhere
\end{align*}
\end{proof}

\subsection{Second moment matrix estimation}
In this section we give some standard results for the rate at which
the empirical correlation matrix approaches the population correlation
matrix for subgaussian random variables.
Let $x\sim\subG_d(\tau^{2})$ be a subgaussian random variable and let
$\Sigma=\En\brk*{xx^{\trn}}$. Let $x_1,\ldots,x_n$ be i.i.d. draws from
$x$, and let $\empcov=\frac{1}{n}\sum_{t=1}^{n}x_tx_t^{\trn}$.
\begin{proposition}
\label{prop:isotropic_covariance_error}
Suppose $x\sim{}\subg_d(\tau^2)$ and $\Sigma=I$. Then with probability at least $1-\delta$,
\begin{equation}
\label{eq:isotropic_covariance_error}
1-\veps \leq{} \eigmin^{1/2}(\empcov) \leq{} \eigmax^{1/2}(\empcov) \leq{} 1+ \veps,
\end{equation}
where $\veps \leq{} 50\tau^{2}\sqrt{\frac{d+\log(2/\delta)}{n}}$.
\end{proposition}
\begin{proof}[\pfref{prop:isotropic_covariance_error}]
Tracking constants carefully, equation (5.23) in
\cite{vershynin2012introduction} 
establishes that with probability at least $1-2e^{-\frac{ct^{2}}{\tau^{4}}}$, 
\[
\nrm*{\empcov-I}_{\specnorm} \leq{} \veps\vee\veps^2,
\]
where $\veps= C\tau^{2}\sqrt{\frac{d}{n}} + \frac{t}{\sqrt{n}}$, and $c>10^{-3}$ and $C\leq{} 25$ are absolute constants. The result follows by \cite{vershynin2012introduction}, Lemma 5.36.
\end{proof}

\begin{proposition}
\label{prop:covariance_error}
Let $\Sigma$ be positive definite, and suppose $\empcov$ is such that 
\[
1-\veps \leq{}
\lambda_{\mathrm{min}}^{1/2}(\Sigma^{-1/2}\empcov\Sigma^{-1/2}) \leq{}
\lambda_{\mathrm{max}}^{1/2}(\Sigma^{-1/2}\empcov\Sigma^{-1/2}) \leq{}
1+ \veps,
\]
where $\veps\leq{}1/2$. Then the following inequality holds:
\[
1-2\veps \leq{}
    \lambda_{\mathrm{min}}^{1/2}(\Sigma^{1/2}\empcov^{-1}\Sigma^{1/2})
    \leq{}
    \lambda_{\mathrm{max}}^{1/2}(\Sigma^{1/2}\empcov^{-1}\Sigma^{1/2})
    \leq{} 1+ 2\veps,
\]
and furthermore, 
\begin{align}
  \nrm*{\empcov-\cov}_{2} &\leq{} \eigmax(\Sigma)\cdot{}3\veps,\\
  \nrm*{\empcov^{-1}-\cov^{-1}}_{2} &\leq{}
                                      \frac{6\veps}{\eigmin(\Sigma)},\\
\nrm*{\cov^{-1/2}\empcov\cov^{-1/2} - I}_2 &\leq{} 3\veps,\\
  \nrm*{\cov^{1/2}\empcov^{-1}\cov^{1/2} - I}_2 &\leq{} 6\veps.
\end{align}
\end{proposition}

\begin{proof}[\pfref{prop:covariance_error}]
To begin, note that the assumed inequality immediately implies that
$\empcov^{-1}$ is well-defined and that
\[
1-2\veps \leq{} \lambda_{\mathrm{min}}^{1/2}(\Sigma^{1/2}\empcov^{-1}\Sigma^{1/2}) \leq{} \lambda_{\mathrm{max}}^{1/2}(\Sigma^{1/2}\empcov^{-1}\Sigma^{1/2}) \leq{} 1+ 2\veps,
\]
using the elementary fact that $(1-\veps)^{-1}\leq{}1+2\veps$ and
$(1+\veps)^{-1}\geq{}1-2\veps$ when $\veps\leq{}1/2$. This inequality
and the assumed inequality, together with Lemma 5.36 of \cite{vershynin2012introduction} imply that
\[
\nrm*{\cov^{-1/2}\empcov\cov^{-1/2} - I}_2 \leq{} 3\veps,\quad\text{and}\quad\nrm*{\cov^{1/2}\empcov^{-1}\cov^{1/2} - I}_2 \leq{} 6\veps.
\]
Finally, observe that we can write
\[
\nrm*{\empcov-\cov}_{2}=\nrm*{\cov^{1/2}(\cov^{-1/2}\empcov\cov^{-1/2} - I)\cov^{1/2}}_2
\leq{} \eigmax(\cov)\cdot \nrm*{\cov^{-1/2}\empcov\cov^{-1/2} - I}_2,
\]
and
\[
\nrm*{\empcov^{-1}-\cov^{-1}}_{2}=\nrm*{\cov^{-1/2}(\cov^{1/2}\empcov^{-1}\cov^{1/2} - I)\cov^{-1/2}}_2
\leq{} \frac{1}{\eigmin(\cov)}\cdot\nrm*{\cov^{1/2}\empcov^{-1}\cov^{1/2} - I}_2.
\]
This establishes the result.
\end{proof}

\subsection{Agnostic contextual bandit algorithm for Natarajan classes (\textsf{Exp4-IX})}
\label{app:cb}
\begin{algorithm}
\begin{algorithmic}
\STATE \textbf{Input:} policy class $\Pi$ with Natarajan dimension $d$, maximum number of rounds $T$, confidence parameter $\delta$.
\FOR{$k=0,1,\ldots$}
\STATE Explore uniformly at random for $n_k = \sqrt{2^k d \log\rbr{\frac{TK}{\delta}}}$ rounds.
\STATE $t\gets{} 2^k+n_k$, $\eta_k \gets{} \sqrt{\frac{d \log\rbr{\frac{TK}{\delta}}}{2^k K}}$.
\STATE Let $\Pi_k \in \Pi$ be the set of unique policies witnessed by $x_{2^k}, \ldots, x_{t-1}$ (choose a representative from each equivalence class).
\STATE Initialize $P_t$ to be the uniform distribution over $\Pi_k$.
\WHILE{$t < 2^{k+1}$}
\STATE Receive $x_t$, sample an action $a_t \sim P_t(\cdot | x_t)$, and observe $\ell_t(a_t)$.
\STATE Compute $P_{t+1}$ such that $P_{t+1}(\pi) \propto P_t(\pi)\exp\rbr{ \frac{-2\eta_k\rbr{\frac{\ell_t(a_t)}{b}+1}\one\cbr{a_t = \pi(x_t)}}{P_t(a_t | x_t)+\eta_k}}, ~\forall \pi \in \Pi_k$.
\STATE $t\gets{} t+1$.
\ENDWHILE
\ENDFOR
\end{algorithmic}
\caption{\textsf{Exp4-IX} with Natarajan class and anytime guarantee}
\label{alg:Exp4}
\end{algorithm}
Here we present a variant of the \textsf{Exp4-IX}
algorithm~\citep{neu2015explore}, originally proposed for achieving
high-probability regret bounds for contextual bandits with a finite
policy class and bounded non-negative losses. There are three main differences in the variant we present here (see \pref{alg:Exp4} for the pseudocode).

First, for our application we need an ``anytime'' regret guarantee that holds for any time $T' \leq T$. This can be simply achieved by a standard doubling trick. 
Specifically, we run the algorithm on an exponential epoch schedule (that is, epoch $k$ lasts for $2^k$ rounds) and restart \textsf{Exp4-IX} at the beginning of each epoch with new parameters.

Second, our policy class is infinite, but with a finite Natarajan
dimension. For the two-action case with a VC class,
\cite{beygelzimer2011contextual} gave a solution to this problem for
stochastic contextual bandits. We extend their approach to Natarajan
classes. Specifically, in epoch $k$ we spend the first $n_k$ rounds collecting contexts (while picking actions arbitrarily).
Then we form a finite policy subset $\Pi_k$ by picking one representative from each equivalence class (that is, all policies that behave exactly the same on these contexts), and play \textsf{Exp4-IX} using this finite policy class $\Pi_k$ for the rest of the epoch.

Finally, in our setting losses are subgaussian and potentially unbounded. However, since with high probability they are bounded essentially by the subgaussian parameter (see \pref{prop:loss_bounded}), we simply pick $b = O\rbr{\tau\log^{1/2}\rbr{\frac{TK}{\delta}}}$ such that with probability at least $1 - \delta/3$, $\max_{a\in\cA,t\in\brk*{T}}\abs*{\ls_t(a)}\leq{} b$, and transform every loss $\ell(a)$ as $\ell(a)/b + 1$, which with high probability falls in $[0, 2]$.
The rest of the algorithm is the same as the original \textsf{Exp4-IX}.
Note that we use the notation $P_t(a|x_t)$ to denote $\sum_{\pi: \pi(x_t)=a}P_t(\pi)$.

The regret guarantee for \pref{alg:Exp4} is as follows.

\begin{proposition}
\label{prop:exp4}
With probability at least $1-\delta$, \pref{alg:Exp4} ensures that for all $T' \in \crl*{1, \ldots, T}$ and $\pi \in \Pi$,
\[
\sum_{t=1}^{T'} \ell_t(a_t) - \ell_t(\pi(x_t)) = O\rbr{\tau\sqrt{T'Kd}\log\rbr{\frac{TK}{\delta}}}
\]
\end{proposition}

\begin{proof}[\pfref{prop:exp4}]
We condition on the event $\max_{a\in\cA,t\in\brk*{T}}\abs*{\ls_t(a)}\leq{} b$ which happens with probability at least $1-\delta/3$.
Following~\citep{neu2015explore}, with probability at least $1-\delta/3$ we have for any $T'$ and any $\pi \in \Pi_k$ (where $k$ is the epoch containing $T'$),
\[
\sum_{t=2^k}^{T'} \ell_t(a_t) - \ell_t(\pi(x_t)) = O\rbr{b n_k + b\sqrt{2^kK\rbr{\log|\Pi_k| + \log\rbr{\frac{TK}{\delta}}}}}.
\]
Sauer's lemma for classes with finite Natarajan dimension~\citep{bendavid1995characterizations, haussler1995generalization} gives $\log|\Pi_k| \leq d\log\rbr{\frac{n_ke(K+1)^2}{2d}}$.
Together with the choice of $n_k$ this gives
\[
\sum_{t=2^k}^{T'} \ell_t(a_t) - \ell_t(\pi(x_t)) = O\rbr{b\sqrt{2^k K d \log\rbr{\frac{TK}{\delta}}}}.
\]
On the other hand, following~\cite{beygelzimer2011contextual} we have with probability at least $1-\delta/3$,
\begin{align*}
\min_{\pi\in \Pi_k} \sum_{t=2^k}^{T'} \ell_t(\pi(x_t)) 
&\leq \min_{\pi\in \Pi} \sum_{t=2^k}^{T'} \ell_t(\pi(x_t)) + O\rbr{\frac{b 2^kd}{n_k}\log\rbr{\frac{TK}{\delta}}} \\
&\leq \min_{\pi\in \Pi} \sum_{t=2^k}^{T'} \ell_t(\pi(x_t)) + O\rbr{b\sqrt{2^k d \log\rbr{\frac{TK}{\delta}}}}.
\end{align*}
Combining these inequalities gives
\[
\sum_{t=2^k}^{T'} \ell_t(a_t) \leq \min_{\pi\in \Pi} \sum_{t=2^k}^{T'} \ell_t(\pi(x_t)) + O\rbr{b\sqrt{2^k Kd \log\rbr{\frac{TK}{\delta}}}}.
\]
Summing up regret in each epoch, using $b =
O\rbr{\tau\log^{1/2}\rbr{\frac{TK}{\delta}}}$, and applying a union
bound leads to the result.
\end{proof}

\subsection{Natarajan dimension for linear policy classes}

\begin{proposition}[\cite{daniely2015multiclass}]
\label{prop:linear_natarajan}
Let $\phi(x,a)\in\bbR^{d}$ be a fixed feature map and consider the policy class
\[
\Pi = \crl*{ x\mapsto{}\argmax_{a\in\brk*{K}}\tri*{\beta,\phi(x,a)}\mid{}\beta\in\bbR^{d}}.
\]
The Natarajan dimension of $\Pi$ is at most $O(d\log{}d)$.
\end{proposition}

\section{Proofs from \pref{sec:model_selection_main}}
\label{app:main}

\subsection{Square loss residual estimation}
In this section we give self-contained results on estimating the
square loss in a linear regression setup, extending the results of
\cite{dicker2014variance} and \cite{kong2018estimating}. Our main
result here is the sample complexity bound for \estimateresidual{}
described in \pref{sec:model_selection_main}

We first recall the abstract setting. We receive pairs $(x_1,y_1),\ldots,(x_n, y_n)$ i.i.d. from a distribution $\cD\in\Delta(\bbR^{d}\times{}\bbR)$, where $x\sim\subg_d(\tau^2)$ and $y\sim\subg(\sigma^2)$. Define $\Sigma\defeq\En\brk*{xx^{\trn}}$, and assume $\Sigma\psdgt{}0$. Let $\betas\in\bbR^{d}$ be the predictor that minimizes prediction error:
\[
\betas\defeq\argmin_{\beta\in\bbR^{d}}\En\prn*{\tri*{\beta,x}-y}^{2}.
\]
Suppose $x$ can be partitioned into features $x=(x^{\expone},x^{\exptwo})$, where $x^{\expone}\in\bbR^{d_1}$ and $x^{\exptwo}\in\bbR^{d_2}$, and $d_1+d_2=d$. We define $\betas_1$ to be the optimal predictor when we regress only onto the features $x^{\expone}$:
\[
\betas_1\defeq\argmin_{\beta\in\bbR^{d_1}}\En\prn*{\tri{\beta,x^{\expone}}-y}^{2}.
\]
Our goal is to estimate the residual error incurred by restricting to the features $x^{\expone}$:
\[
\cE\defeq\En\prn*{\tri{\betas_1,x^{\expone}}-\tri*{\betas,x}}^{2}.
\]
\estimateresidual (\pref{alg:estimator})
 estimates $\cE$ with sample complexity
sublinear in $d$ whenever good estimates for the matrices $\Sigma$ and
$\Sigma_1\defeq\En\brk*{x^{\expone}x^{\expone\trn}}$ are
available.\footnote{Note that $\Sigma_1\psdgt{}0$.}
The performance is stated in \pref{thm:residual_estimation}. The result here is a more general version
of~\pref{thm:residual_est_main}, which is
proven as a corollary at the end of the section.

\begin{theorem}
\label{thm:residual_estimation}
Suppose the correlation matrix estimates $\empcov$ and $\empcov_1$ are such that
\[
1-\veps \leq{} \lambda_{\mathrm{min}}^{1/2}(\Sigma^{-1/2}\empcov\Sigma^{-1/2}) \leq{} \lambda_{\mathrm{max}}^{1/2}(\Sigma^{-1/2}\empcov\Sigma^{-1/2}) \leq{} 1+ \veps,
\]
and
\[
1-\veps \leq{} \lambda_{\mathrm{min}}^{1/2}(\Sigma_1^{-1/2}\empcov_1\Sigma^{-1/2}_1) \leq{} \lambda_{\mathrm{max}}^{1/2}(\Sigma^{-1/2}_1\empcov_1\Sigma^{-1/2}_1) \leq{} 1+ \veps,
\]
where $\veps\leq{}1/2$. Then \estimateresidual guarantees that with probability at least
$1-\delta$,
\begin{equation}
  \label{eq:residual_estimation}
  \abs*{\wh{\cE}-\cE} \leq{} \frac{1}{2}\cE 
+
O\prn*{\frac{\eigmax(\cov)}{\eigmin^{2}(\cov)}\cdot\frac{\sigma^{2}\tau^{2}d^{1/2}\log^{2}(2d/\delta)}{n}
+ 
\frac{\eigmax(\cov)}{\eigmin^{2}(\cov)}\nrm*{\En\brk*{xy}}_{2}^{2}\cdot\veps^{2}
}.
\end{equation}
\end{theorem}
\wontfix{Note to self: $\log^{2}(d/\delta)$ can be tightened to
  $\log(d/\delta)$ once $n \geq{} \log(d/\delta)$.}

\begin{proof}[\pfref{thm:residual_estimation}]
We begin by giving an expression for $\cE$ that will make the choice
of estimator more transparent. Let \[D= \left(\begin{array}{ll}
\cov_1&0_{d_1\times{}d_2}\\
0_{d_2\times{}d_1}&0_{d_2\times{}d_2}\end{array}
\right),\] and let $R=D^{\pinv}-\Sigma^{-1}$. Observe that by first-order
conditions, we may take $\betas=\Sigma^{-1}\En\brk*{xy}$ and
$\betas_1=\Sigma_1^{-1}\En\brk*{x^{\expone}y}$. Moreover, for any $x$,
we may write
$\tri*{\betas_1,x^{\expone}}=\tri*{D^{\pinv}\En\brk*{xy},x}$. Consequently,
we have
\begin{align*}
\cE&=\En\prn*{\tri{D^{\pinv}\En\brk*{xy},x}-\tri*{\Sigma^{-1}\En\brk*{xy},x}}^{2} 
=\En\tri{R\En\brk*{xy},x}^{2} 
=\nrm*{\Sigma^{1/2}R\En\brk*{xy}}_{2}^{2}.
\end{align*}
With this representation, it is clear that if $\empcov=\cov$ and $\empcov_1=\cov_1$, then
$\wh{\cE}$ is an unbiased estimator for $\cE$. Our proof has two
parts: We first show that $\wh{\cE}$ concentrates around its
expectation, then bound the bias due to $\empcov$ and $\empcov_1$.

For concentration, we appeal to
\pref{lem:ustatistic_subgaussian}. 
Note that $\eigmin(\cov_1)\geq{}\eigmin(\cov)$ and
$\eigmax(\cov_1)\leq{}\eigmax(\cov)$; this can be seen using the
variational representation for the eigenvalues. Consequently by \pref{prop:covariance_error} we have that
\begin{align*}
  \eigmax(\empcov)\vee\eigmax(\empcov_1)&\leq{}O(\eigmax(\cov)),\quad\text{and},\quad
  \eigmin(\empcov)\wedge\eigmin(\empcov_1)\geq{}\Omega\prn*{\eigmin(\cov)}.
\end{align*}
This implies that  
\[
\nrm*{\empcov^{1/2}\wh{R}}_{\specnorm}
\leq{}O\prn*{\eigmax^{1/2}(\cov)/\eigmin(\cov)},
\]
and so by \pref{prop:subgaussian_closure}, it follows that the random
variable $\empcov^{1/2}\wh{R}xy-\En\brk*{\empcov^{1/2}\wh{R}xy}$ has
subexponential parameter of order $O(\sigma\tau \eigmax^{1/2}(\cov)/\eigmin(\cov))$. Thus, by
\pref{lem:ustatistic_subgaussian}, we have that with probability at
least $1-\delta$,
\[
\abs*{\wh{\cE}-\En\brk*{\wh{\cE}}} \leq{}
O\prn*{\frac{\eigmax(\cov)}{\eigmin^{2}(\cov)}\frac{\sigma^{2}\tau^{2} d^{1/2}\log^{2}(2d/\delta)}{n} +
\frac{\eigmax^{1/2}(\cov)}{\eigmin(\cov)}\frac{\sigma\tau\nrm*{\empcov^{1/2}\wh{R}\En\brk*{xy}}_{2}\log(2/\delta)}{\sqrt{n}}
}.
\]
But note that since
$\nrm*{\empcov^{1/2}\wh{R}\En\brk*{xy}}_{2}=\sqrt{\En\brk*{\wh{\cE}}}$,
we can apply the AM-GM inequality to deduce
\[
\abs*{\wh{\cE}-\En\brk*{\wh{\cE}}} \leq{}
\frac{1}{8}\En\brk*{\wh{\cE}} + 
O\prn*{\frac{\eigmax(\cov)}{\eigmin^{2}(\cov)}\cdot\frac{\sigma^{2}\tau^{2}d^{1/2}\log^{2}(2d/\delta)}{n}}.
\]
We now bound the error from $\En\brk*{\wh{\cE}}$ to $\cE$. 
With the shorthand $\mu=\En\brk*{xy}$, observe that we have
\begin{align*}
\En\brk*{\wh{\cE}}-\cE
&= \tri*{\wh{R}\empcov\wh{R}\mu,\mu}-\tri*{R\cov{}R\mu,\mu} \\
&=2\tri*{(\wh{R}-R)\empcov{}R\mu,\mu}
  + \tri*{R(\empcov-\cov)R\mu,\mu} + \tri*{(\wh{R}-R)\empcov (\wh{R}-R)\mu,\mu}.
\end{align*}
Applying Cauchy-Schwarz to each term, we get an upper bound of
\begin{align*}
\abs*{\En\brk*{\wh{\cE}}-\cE}&\leq
  2\nrm*{\Sigma^{-1/2}\empcov(R-\wh{R})\mu}_{2}\nrm*{\Sigma^{1/2}R\mu}_{2}
+
  \nrm*{\Sigma^{-1/2}(\cov-\empcov)R\mu}_{2}\nrm*{\Sigma^{1/2}R\mu}_{2} \\
&~~~~+ \nrm*{\empcov^{1/2} (R-\wh{R})\mu}_{2}^{2}.
\end{align*}
Since $\nrm*{\Sigma^{1/2}R\mu}_{2}=\sqrt{\cE}$, we can apply the AM-GM
inequality to each of the first two terms to conclude that
\begin{align}
\label{eq:variance_est_bias} &\abs*{\En\brk*{\wh{\cE}}-\cE} 
\leq \frac{1}{8}\cE
+ O\prn*{
\nrm*{\Sigma^{-1/2}\empcov(R-\wh{R})\mu}_{2}^{2}
+ \nrm*{\Sigma^{-1/2}(\cov-\empcov)\wh{R}\mu}_{2}^{2}
+ 
\nrm*{\empcov^{1/2} (R-\wh{R})\mu}_{2}^{2}
  }.
\end{align}
To proceed, we first collect a number of spectral properties, all of
which follow from the assumptions in the theorem statement and \pref{prop:covariance_error}:
\begin{align*}
  \nrm*{\cov^{-1/2}\empcov}_{\specnorm}
= \nrm*{\cov^{-1/2}\empcov\cov^{-1/2}\cov^{1/2}}_{\specnorm}\leq{}(1+\veps)\nrm*{\cov^{1/2}}_{\specnorm}
&\leq{}O(\eigmax^{1/2}(\cov)),\\
  \nrm*{\wh{R}}_{\specnorm}&\leq{}O\prn*{1/\eigmin(\cov)},\\
  \nrm*{R-\wh{R}}_{\specnorm}
\leq{} \nrm*{\empcov_{1}^{-1}-\cov_{1}^{-1}}_{\specnorm} +
    \nrm*{\empcov^{-1}-\cov^{-1}}_{\specnorm}
                                       &\leq{}O\prn*{\veps/\eigmin(\cov)},\\
\nrm*{\Sigma^{-1/2}\empcov\Sigma^{-1/2}-I}_{2} &\leq{} O(\veps).
\end{align*}
We now bound the terms in \pref{eq:variance_est_bias} one by
one. Using the spectral bounds above, we have
\begin{align*}
  \nrm*{\Sigma^{-1/2}\empcov(R-\wh{R})\mu}_{2}^{2}
\leq{}\nrm*{\Sigma^{-1/2}\empcov}_{\specnorm}^{2}\nrm*{\mu}_{2}^{2}\cdot\nrm*{R-\wh{R}}_{\specnorm}^{2}
&\leq{} O\prn*{
\frac{\eigmax(\cov)}{\eigmin^{2}(\cov)}\nrm*{\mu}_{2}^{2}\cdot\veps^{2}
},\\
\nrm*{\Sigma^{-1/2}(\cov-\empcov)\wh{R}\mu}_{2}^{2}
\leq{}\nrm*{\cov}_{2}\nrm*{\wh{R}}_{2}^{2}\nrm*{\mu}_{2}^{2}\cdot
  \nrm*{\Sigma^{-1/2}\empcov\Sigma^{-1/2}-I}_{2}^{2}
&\leq{} O\prn*{
\frac{\eigmax(\cov)}{\eigmin^{2}(\cov)}\nrm*{\mu}_{2}^{2}\cdot\veps^{2}
},\\
  \nrm*{\empcov^{1/2}(R-\wh{R})\mu}_{2}^{2}
\leq{}\nrm*{\empcov}_{\specnorm}\nrm*{\mu}_{2}^{2}\cdot\nrm*{R-\wh{R}}_{\specnorm}^{2}
&\leq{} O\prn*{
\frac{\eigmax(\cov)}{\eigmin^{2}(\cov)}\nrm*{\mu}_{2}^{2}\cdot\veps^{2}
}.
\end{align*}
We conclude that
\[
\abs*{\En\brk*{\wh{\cE}}-\cE}
\leq{} \frac{1}{8}\cE + O\prn*{
\frac{\eigmax(\cov)}{\eigmin^{2}(\cov)}\nrm*{\mu}_{2}^{2}\cdot\veps^{2}
},
\]
which yields the result.
\end{proof}

\begin{lemma}
\label{lem:ustatistic_subgaussian}
Let $z$ be a random variable such that $z-\En\brk*{z}\sim\subE_d(\lambda)$, and let $z_1,\ldots,z_n$ be i.i.d. copies.
Define 
\[
W=\sum_{i<j}\tri*{z_i,z_j}.
\]
Then with probability at least $1-\delta$,
\[
\abs*{W-\En\brk*{W}} \leq{} O\prn*{\lambda^{2}d^{1/2}n\log^{2}(2d/\delta) + \lambda{}n^{3/2}\nrm*{\En\brk*{z}}_{2}\log(2/\delta)}.
\]

\end{lemma}
\begin{proof}[\pfref{lem:ustatistic_subgaussian}]
 First observe that we can write
\begin{align*}
\abs*{W-\En\brk*{W}}&=\abs*{
\sum_{i<j}\tri{z_i,z_j} - \tri*{\En\brk*{z},\En\brk*{z}}
} \\
&=\abs*{
\sum_{i<j}\tri{z_i-\En\brk*{z},z_j-\En\brk*{z}} + \sum_{i<j}\tri{z_i-\En\brk*{z},\En\brk*{z}} + \sum_{i<j}\tri{z_j-\En\brk*{z},\En\brk*{z}} } \\
&\leq{}\underbrace{\abs*{
\sum_{i<j}\tri{z_i-\En\brk*{z},z_j-\En\brk*{z}}
}}_{S_1}
+
\underbrace{(n-1)\abs*{
\sum_{i=1}^{n}\tri{z_i-\En\brk*{z},\En\brk*{z}}
}}_{S_2}.
\end{align*}

We bound $S_2$ first. Define $B=\nrm*{\En\brk*{z}}_2$. Observe that for each $i$, the summand is subexponential: $\tri*{z_i-\En\brk*{z}, \En\brk*{z}}\sim{}\subE(\lambda{}B)$. Consequently, Bernstein's inequality implies that with probability at least $1-\delta$,
\[
S_2 \leq{} O\prn*{\lambda{}Bn^{3/2}\prn*{\log(2/\delta)}^{1/2} + 2\lambda{}Bn\log(2/\delta)}.
\]
For $S_1$, we first apply a decoupling inequality. Let
$z'_1,\ldots,z'_n$ be a sequence of independent copies of
$z_1,\ldots,z_n$. Then by Theorem 3.4.1 of
\cite{delapena1998decoupling}, there are universal constants $C,c>0$
such that for any $t\geq{}0$,
\[
\Pr(S_1>t) \leq{} C\Pr\prn*{
\abs*{
\sum_{i<j}\tri{z_i-\En\brk*{z},z'_j-\En\brk*{z}}
}>t/c
}.
\]
We write
\[
\sum_{i<j}\tri{z_i-\En\brk*{z},z'_j-\En\brk*{z}}
=\sum_{i=1}^{n}\tri*{z_i-\En\brk*{z},\sum_{j=i+1}^{n}z'_j-\En\brk*{z}}
= \sum_{i=1}^{n}\tri*{z_i-\En\brk*{z},Z_i},
\]
where $Z_i=\sum_{j=i+1}^{n}z'_j-\En\brk*{z}$.
Now condition on $z'_1,\ldots,z'_n$. Then $\tri*{z_i-\En\brk*{z},Z_i}\sim\subE(\lambda\nrm*{Z_i}_{2})$. Thus, by Bernstein's inequality, we have that with probability at least $1-\delta$, over the draw of $z_1,\ldots,z_n$,
\[
\abs*{\sum_{i<j}\tri{z_i-\En\brk*{z},z'_j-\En\brk*{z}}}
\leq{} O\prn*{\lambda{}\prn*{n\log(2/\delta)}^{1/2} + \lambda{}\log(2/\delta)}\cdot\max_{i\in\brk*{n}}\nrm*{Z_i}_{2}.
\]
Next, using \pref{lem:subexponential_max}, we have that with probability at least $1-\delta$,
\[
\max_{i\in\brk*{n}}\nrm*{Z_i}_{2}
\leq{} O\prn*{\lambda\sqrt{dn\log(2d/\delta)} + \lambda{}d^{1/2}\log(2d/\delta)}.
\]
Thus, by union bound, after grouping terms we get that with probability at least $1-\delta$,
\[
\abs*{\sum_{i<j}\tri{z_i-\En\brk*{z},z'_j-\En\brk*{z}}}\leq{} O\prn*{\lambda^{2}d^{1/2}n\log^{2}(2d/\delta)},
\]
and
\[
S_2 \leq{} O\prn*{\lambda^{2}d^{1/2}n\log^{2}(2d/\delta)}.
\]
Taking a union bound yields the final result.
\end{proof}

\begin{proof}[\pfref{thm:residual_est_main}]
Consider the distribution over random vectors $x'\ldef{}\Sigma^{-1/2}x$, and let $\empcov'$ denote the empirical covariance under this distribution. Since $\En\brk*{x'x'^{\trn}}=I$, we may apply \pref{prop:isotropic_covariance_error}, which implies that with probability at least $1-\delta$,
\[
1-\veps \leq{} \lambda_{\mathrm{min}}^{1/2}(\empcov') \leq{} \lambda_{\mathrm{max}}^{1/2}(\empcov') \leq{} 1+ \veps,
\]
where $\veps \leq{}
50\eigmin^{-1}(\cov)\tau^{2}\sqrt{\frac{d+\log(2/\delta)}{m}}$, and we
have used that the subgaussian parameter of $x'$ is at most
$\eigmin^{-1}(\cov)$ times that of $x$. We can equivalently write this
expression as
\[
1-\veps \leq{} \lambda_{\mathrm{min}}^{1/2}(\Sigma^{-1/2}\empcov\Sigma^{-1/2}) \leq{} \lambda_{\mathrm{max}}^{1/2}(\Sigma^{-1/2}\empcov\Sigma^{-1/2}) \leq{} 1+ \veps.
\]
Note also that once
$m\geq{}C(d+\log(2/\delta))\tau^{4}/\eigmin(\Sigma)$ for some universal
constant $C$, we have $\veps\leq{}1/2$. Applying the same reasoning to
$\empcov_1$ and taking a union bound, then appealing to
\pref{thm:residual_estimation} yields the result. We use that
$\eigmax(\cov)\leq{}\tau^{2}$ to simplify the final expression.
\end{proof}

\subsection{Proof of \pref{thm:model_selection_main}}

\subsubsection{Basic technical results}
In this section we prove some utility results that bound the
scale of various random variables appearing in the analysis of \mainalgo{}.

\begin{proposition}
\label{prop:loss_subgaussian}
For all $a$, $\ls(a)\sim\subG(4\tau^{2})$.  
\end{proposition}
\begin{proof}
We have $\ls(a) = \tri{\betas,\phi^{\ms}(x)} + \veps_a$, where $\veps_a=\ell(a) - \EE[\ell(a)\mid x]$. Note that $\tri{\betas,\phi^{\ms}(x)}$ has subgaussian parameter $B^{2}\tau^{2}$, and $\veps_a$ has subgaussian parameter $\sigma^{2}$, so the triangle inequality for subgaussian parameters implies that the parameter of $\ls(a)$ is at most $(B\tau+\sigma)^{2}\leq{}4\tau^{2}$, where we have used the assumption that $B\leq{}1$ and $\sigma\leq{}\tau$. 
\end{proof}
\begin{proposition}
\label{prop:loss_bounded}
With probability at least $1-\delta$,
$\max_{a\in\cA,t\in\brk*{T}}\abs*{\ls_t(a)}\leq{}O\prn*{\tau\sqrt{\log(KT/\delta)}}$. 
\end{proposition}
\begin{proof}
  Immediate consequence of \pref{prop:loss_subgaussian}, along with Hoeffding's inequality and a union bound.
\end{proof}

\subsubsection{Square loss translation}
We first introduce some additional notation. Let
$\Ls_m=\min_{\pi\in\Pi_m}L(\pi)$. Recall that
\begin{align*}
\betas_m=\argmin_{\beta\in\bbR^{d_m}}\frac{1}{K}\sum_{a\in\cA}\En_{x\sim\cD}\prn*{\tri*{\beta,\phi^{m}(x,a)}-\ls(a)}^{2}. 
\end{align*}
We
let $\pisq_m(x)=\argmin_{a\in\cA}\tri*{\betas_m,\phi^{m}(x,a)}$ be the
induced policy.

\begin{proposition}
  \label{prop:beta_basics}
  For all $m\in\brk*{M}$, $\betas_m$ is uniquely defined, and $\betas_m=(\betas,\mb{0}_{d_m-d_{\ms}})$ for all $m\geq{}\ms$. As a consequence:
  \begin{itemize}[topsep=0pt,partopsep=0pt,parsep=0pt]
  \item $\pisq_m=\index{$\pis$}\pis$ for all $m\geq{}\ms$.
  \item $\cE_{i,\ms}=\cE_{i,j}$ for all $i\in\brk*{M}$ and all $j\geq\ms$.
  \item $\cE_{i,j}=0$ for all $j>i\geq{}\ms$.
  \end{itemize}
\end{proposition}
\begin{proof}[\pfref{prop:beta_basics}]
That each $\betas_{m}$ is uniquely defined follows from the assumption that $\eigmin(\cov_m)>0$, since this implies that the optimization problem \pref{eq:square_loss_min} is strongly convex.

To show that $\betas_m=(\betas,\mb{0}_{d_m-d_{\ms}})$ when $m\geq{}\ms$, observe that by first order conditions, $\betas_m$ is uniquely defined via
\[
\textstyle\betas_m=\cov_m^{-1}\cdot \frac{1}{K}\sum_{a\in\cA}\En_{(x,\ls)\sim\cD}\brk*{\phi^{m}(x,a)\ls(a)}.
\]
But note that for each $a\in\cA$, the realizability assumption \pref{eq:realizable} implies that
\begin{align*}
\En_{(x,\ls)\sim\cD}\brk*{\phi^{m}(x,a)\ls(a)}&=\En_{x\sim\cD}\brk*{\phi^{m}(x,a)\tri{\phi^{\ms}(x,a),\betas}}\\
&=\En_{x\sim\cD}\brk*{\phi^{m}(x,a)\phi^{m}(x,a)^{\trn}(\betas,\mb{0}_{d_{m}-d_{\ms}})},
\end{align*}
where the last equality follows from the nested feature map assumption. Combining this with the preceding identity, we have
\[
\textstyle\betas_m=\cov_m^{-1}\cov_{m}(\betas,\mb{0}_{d_{m}-d_{\ms}}) = (\betas,\mb{0}_{d_{m}-d_{\ms}}).
\]
The remaining claims now follow immediately from the definition of $\pisq$ and $\cE_{i,j}$.
\end{proof}

\begin{proposition}
\label{prop:norm_bounds}
For all $a\in\cA$ and $m\in\brk*{M}$, we have
$\nrm*{\En\brk*{\phi^{m}(x,a)\ls(a)}}_{2}\leq{}\tau^{2}$. For all
$m\in\brk*{M}$, we have $\nrm*{\betas_m}_{2}\leq{}\tau/\gamma$, and so $\pisq_m\in\Picompact_m$.
\end{proposition}
\begin{proof}
For the first claim, we use realizability to write
\[
\En\brk*{\phi^{m}(x,a)\ls(a)}
= \En\brk*{\phi^{m}(x,a)\tri{\phi^{\ms}(x,a),\betas}}
\]
Hence any $\theta\in\bbR^{d_m}$ with $\nrm*{\theta}_{2}\leq{}1$, we
have
\begin{align*}
\tri*{\En\brk*{\phi^{m}(x,a)\ls(a)},\theta} &= 
\En\brk*{\tri{\phi^{m}(x,a),\theta}\tri{\phi^{\ms}(x,a),\betas}}\\
&\leq{} \sqrt{\En\tri{\phi^{m}(x,a),\theta}^{2}\cdot \En\tri{\phi^{\ms}(x,a),\betas}^{2}}\\&\leq{}\tau^{ 2},
\end{align*}
where the first inequality is Cauchy-Schwarz and the second follows
because $\nrm*{\betas}_{2}\leq{}1$ by assumption and all feature maps
belong to $\subG_{d_m}(\tau^{2})$.
For the second claim, recall as in the proof of
\pref{prop:beta_basics} that $\betas_m$ is uniquely defined as
\[
\textstyle\betas_m=\cov_m^{-1}\cdot
\frac{1}{K}\sum_{a\in\cA}\En_{(x,\ls)\sim\cD}\brk*{\phi^{m}(x,a)\ls(a)}
= 
\frac{1}{K}\sum_{a\in\cA}\En_{(x,\ls)\sim\cD}\brk*{\cov_m^{-1}\phi^{m}(x,a)\tri{\phi^{\ms}(x,a),\betas}}.
\]
Following the same approach as for the first claim, for any any
$\theta\in\bbR^{d_m}$ with $\nrm*{\theta}_{2}\leq{}1$, we have
\begin{align*}
\textstyle
  \tri*{\theta,\betas_m}
&=
  \frac{1}{K}\sum_{a\in\cA}\En_{(x,\ls)\sim\cD}\brk*{\tri*{\cov_m^{-1}\phi^{m}(x,a),\theta}\tri{\phi^{\ms}(x,a),\betas}} \\
&\leq
  \sqrt{\frac{1}{K}\sum_{a\in\cA}\En_{(x,\ls)\sim\cD}\tri*{\cov_m^{-1}\phi^{m}(x,a),\theta}^{2}\cdot
\frac{1}{K}\sum_{a\in\cA}\En_{(x,\ls)\sim\cD}\tri{\phi^{\ms}(x,a),\betas}^{2}}
  \\
&=
  \sqrt{\tri*{\Sigma_{m}^{-1}\theta,\theta}\cdot
\frac{1}{K}\sum_{a\in\cA}\En_{(x,\ls)\sim\cD}\tri{\phi^{\ms}(x,a),\betas}^{2}}\\
&\leq{} \tau/\gamma,
\end{align*}
where we again have used that $\nrm*{\betas}_{2}\leq{}1$.
\end{proof}

\begin{lemma}
  \label{lem:square_loss_translation}
  For all $i\in\brk*{M}$, and all $j\geq{}\ms$,
  \begin{equation}
    \label{eq:square_loss_translation}
    \Delta_{i,j}\leq{}L(\pisq_i) - L(\pisq_j)\leq{}\sqrt{4K\cdot\cE_{i,j}}.
  \end{equation}
\end{lemma}
\begin{proof}[\pfref{lem:square_loss_translation}]
Observe that we have
\[
\Ls_i-\Ls_j \leq{} L(\pisq_i) - L(\pi^{\star}_j) = L(\pisq_i) - L(\pisq_j),
\]
where the inequality holds because $\pisq_i\in\Picompact_i$ and the
equality follows from \pref{prop:beta_basics} and  the assumption that $j\geq\ms$. Using realizability along with the representation for $\betas_j$ from \pref{prop:beta_basics}, we write
\[
L(\pisq_i) - L(\pisq_j)=\En_{x\sim\cD}\brk*{
\tri*{\betas_j,\phi^{j}(x,\pisq_i(x))}-
\tri*{\betas_j,\phi^{j}(x,\pisq_j(x))}
}
\]
For each $x$, we have
\[
\tri*{\betas_i,\phi^{i}(x,\pisq_j(x))}-\tri*{\betas_i,\phi^{i}(x,\pisq_i(x))}\geq{}0,
\]
which follows from the definition of $\pisq_i$. We add this inequality to the preceding equation to get
\begin{align*}
L(\pisq_i) - L(\pisq_j)&\leq{}\En_{x\sim\cD}\brk*{
\tri*{\betas_j,\phi^{j}(x,\pisq_i(x))}-\tri*{\betas_i,\phi^{i}(x,\pisq_i(x))}}\\
&~~~~+\En_{x\sim\cD}\brk*{\tri*{\betas_i,\phi^{i}(x,\pisq_j(x))}-\tri*{\betas_j,\phi^{j}(x,\pisq_j(x))}
}\\
&\leq{} 2\En_{x\sim\cD}\max_{a}\abs*{
\tri*{\betas_i,\phi^{i}(x,a)}-\tri*{\betas_j,\phi^{j}(x,a)}
}.
\end{align*}
Lastly, using Jensen's inequality, we have
\begin{align*}
  \En_{x\sim\cD}\max_{a}\abs*{
\tri*{\betas_i,\phi^{i}(x,a)}-\tri*{\betas_j,\phi^{j}(x,a)}
}
&\leq{} \textstyle
\sqrt{\En_{x\sim\cD}\max_{a}\prn*{
\tri*{\betas_i,\phi^{i}(x,a)}-\tri*{\betas_j,\phi^{j}(x,a)}
}^{2}}\\
&\leq{} \textstyle
\sqrt{K\cdot{}\frac{1}{K}\sum_{a\in\cA}\En_{x\sim\cD}\prn*{
\tri*{\betas_i,\phi^{i}(x,a)}-\tri*{\betas_j,\phi^{j}(x,a)}
}^{2}}\\
&=\sqrt{K\cdot\cE_{i,j}}.\tag*\qedhere
\end{align*}
\end{proof}

\subsubsection{Decomposition of regret}
To proceed with the analysis we require some additional notation. We let $N$ be the number of values that
$\mh$ takes on throughout the execution of the algorithm (i.e., $N$ is
the number of candidate policy classes that are tried), and let $\mh_k$ for $k\in\brk*{N}$ denote the $k$th such
value. We let $\cI_k\subseteq\brk*{T}$ denote the interval for which
$\mh=\mh_k$, and let $T_k$ denote the first timestep in this interval.

We let $S_t$ denote the value of the set $S$ at step $t$ (after
uniform exploration has occurred, if it occurred). We let
$\cIbar_k=\cI_k\setminus{}S_T$ denote the rounds in interval $k$ in
which uniform exploration did not occur.

We let
$\wh{\cE}_{ij}(t)$ the random variable defined by running
\estimateresidual using the dataset
$H_j(t)=\crl*{(\phi^{j}(x_s,a_s),\ls_s(a
_s))}_{s\in{}S_t}$ and empirical
second moment matrices $\empcov_i$ and $\empcov_j$ at time $t$. Note
that $\wh{\cE}_{i,j}(t)$ is well-defined even for pairs $(i,j)$ for which
the algorithm does not invoke \estimateresidual at time $t$.

We partition the intervals as follows: Let $k_0$ be the first interval
containing $t\geq{}T_{\ms}^{\min}$, and let $k_{1}$ be the first
inteval for which $k\geq{}\ms$. We will eventually show that with high
probability $k_1\geq{}N$, or in
other words, once the algorithm reaches a class containing the optimal
policy it never leaves (if it reaches such a class, that is).

Let $\Reg(\cI_k)$ denote the regret to $\ms$ incurred throughout
interval $k$. We bound regret to $\pis$ as
\begin{equation}
\Reg\leq{} \sum_{k=1}^{k_0-1}\Reg(\cI_k) +
\sum_{k=k_0}^{k_1-1}\Reg(\cI_k) +
\sum_{k=k_1}^{N}\Reg(\cI_k).\label{eq:regret_decomposition}
\end{equation}

The main result in this subsection is \pref{lem:good_event} which
shows that with high probability, the estimators $\wh{\cE}_{i,j}$ and
\expix{} instances invoked by the algorithm behave as
expected, and various quantities arising in the regret analysis are
bounded appropriately.
\begin{lemma}
\label{lem:good_event}
Let $A$ be the event that the following properties hold:
\begin{enumerate}
\item For all $a\in\cA$ and $t\in\brk*{T}$,
  $\abs*{\ls_t(a)}\leq{}O\prn*{\tau\sqrt{\log(KT/\delta_0)}}$.\hfill \textnormal{(event $A_1$)}
\item For all $t\geq{}\Tmin_1$,
  $\frac{1}{8}K^{\expexp}t^{1-\expexp}\leq{}\abs*{S_t}\leq{}4K^{\expexp}t^{1-\expexp}$. \hfill
  \textnormal{(event $A_2$)}
\item For all $i < j$, for all $t\geq{}\Tmin_j$,
  $\abs*{\wh{\cE}_{i,j}(t)-\cE_{i,j}}\leq{}\frac{1}{2}\cE_{i,j} +
  \alpha_{j,t}$.  \hfill
  \textnormal{(event $A_3$)}
\item For all $k$,
$
\sum_{t\in\cIbar_k}\ls_t(a_t)-\ls_t(\pisq_{\mh_k}(x_t))
\leq{}O\prn*{\tau\sqrt{d_{\mh_k}\abs*{\cIbar_k}\cdot{}K\log^{2}(TK/\delta_0)\log{}d_{\mh_k}}}$.\\
\hspace*{0pt}\hfill
  \textnormal{(event $A_4$)}
\item For all $i,j\in\brk*{M}$ and all intervals $\cI=\brk*{t_1,t_2}$,
\[
\sum_{t\in\cI}\ls_t(\pisq_i(x_t))-\ls_t(\pisq_j(x_t)) \leq{}
\abs*{\cI}\cdot{}(L(\pisq_i) - L(\pisq_j)) +O\prn*{
\tau\sqrt{\abs*{\cI}\log(2/\delta_0)}
}.
\]
\hfill
  \textnormal{(event $A_5$)}
\end{enumerate}
When $C_1$ and $C_2$ are sufficiently large constants, event $A$ holds
with probability at least $1-\delta$.
\end{lemma}
\begin{proof}[\pfref{lem:good_event}]
First, note that event $A_1$ holds with probability at least
$1-\delta/10$ by \pref{prop:loss_bounded}. 

We now move on to $A_2$. For any fixed $t$, Bernstein's inequality implies that with
probability at least $1-\delta_0$,
\[
\abs*{S_t} \geq{} \En_t\abs*{S_t} - \sqrt{4
  \En\abs*{S_t}\log(2/\delta_0)} - \log(2/\delta_0) \geq{}
\frac{1}{2}\En\abs*{S_k} - 3\log(2/\delta_0),
\]
and likewise implies that $\abs*{S_t}\leq{}\frac{3}{2}\En\abs*{S_t} +
3\log(2/\delta_0)$. Next, note that
\[
\textstyle\En\abs*{S_t}=\sum_{s=1}^{t}\prn*{1\wedge\frac{K^{\expexp}}{s^{\expexp}}}.
\]
It follows that
$\En\abs*{S_t}\leq{}2K^{\expexp}t^{1-\expexp}$.\wontfix{I didn't
  check this constant :p} We also
have $\En\abs*{S_t}\geq{}K^{\expexp}(t^{1-\expexp}-2K^{1-\expexp})$,
which is lower bounded by $K^{\expexp}t^{1-\expexp}/2$ once
$t\geq{}4^{\frac{1}{1-\expexp}}K$, and in particular once
$t\geq{}\Tmin_1$ whenever $C_2$ is sufficiently large.
If we union bound over all
$t$, these results together imply that once
$t\geq{}(30\log(2/\delta_0)/K^{\expexp})^{\frac{1}{1-\expexp}}$ (which is implied by
$t\geq{}\Tmin_1$ when $C_2$ is large enough), then with probability at least
$1-\delta_0T\geq{}1-\delta/10$, for all $t$,
\[
\frac{1}{8}K^{\expexp}t^{1-\expexp}\leq{}\abs*{S_t}\leq{}4K^{\expexp}t^{1-\expexp}.
\]
For $A_3$, let $t$ and $i<j$ be fixed. Note that conditioned on the
size of $\abs*{S_t}$, the examples
$\crl*{\phi^{j}(x_s,a_s),\ls_s(a_s))}_{s\in{}S_t}$ are
i.i.d. Consequently, \pref{thm:residual_est_main} implies that with probability at least $1-\delta_0$,
\[
  \abs*{\wh{\cE}_{i,j}(t)-\cE_{i,j}} \leq{} \frac{1}{2}\cE_{i,j} 
+
O\prn*{\frac{\tau^{6}}{\gamma^{4}}\cdot\frac{d_j^{1/2}\log^{2}(2d_j/\delta_0)}{\abs*{S_t}}
+ 
\frac{\tau^{10}}{\gamma^{8}}\cdot\frac{d_j\log(2/\delta_0)}{t}
},
\]
where we have used \pref{prop:loss_subgaussian} to show that
$\ls_s(a_s)\sim\subG(4\tau^{2})$ and used
\pref{prop:norm_bounds} to show that
$\nrm*{\frac{1}{K}\sum_{a\in\cA}\En\brk*{\phi^{j}(x,a)\ls(a)}}_{2}\leq{}\tau^{2}$. Conditioned
on $A_2$, we have $\abs*{S_t}=\Omega(K^{\expexp}t^{1-\expexp})$ once
$t\geq{}\Tmin_1$, and so when $C_1$ is a sufficiently large absolute constant,
$\abs*{\wh{\cE}_{i,j}-\cE_{i,j}} \leq{} \frac{1}{2}\cE_{i,j} +
\alpha_{j,t}$. By union bound, we get that conditioned on event $A_2$,
event $A_3$ holds with probability at least $1-M^{3}T\delta_0\geq{}1-\delta/10$.

To prove $A_4$, we appeal to \pref{prop:exp4}. To do so, we
verify the following facts: 1. Losses belong to $\subG(4\tau^{2})$
(\pref{prop:loss_subgaussian}) 2. Each policy class $\Picompact_m$ is
compact and contains $\pisq_m$ (compactness is immediate, containment
of $\pisq_m$ follows from \pref{prop:norm_bounds}) 3. The Natarajan
dimension of $\Picompact_m$ is at most $O(d_m\log(d_m))$
(\pref{prop:linear_natarajan}).

Now let $k\in\brk*{N}$ be fixed. Since \pref{prop:exp4} provides
an anytime regret guarantee for \expix{}, and since the
context-loss pairs fed into the algorithm still follow the
distribution $\cD$ (the step at which we perform uniform exploration
does not alter the distribution). Consequently, conditioned on the
history up until time $T_k$, we have that with probability at
least
$1-\delta_0$,
\[
\sum_{t\in\cIbar_k}\ls_t(a_t)-\ls_t(\pisq_{\mh_k}(x_t))
\leq{}O\prn*{\tau\sqrt{d_{\mh_k}\abs*{\cIbar_k}\cdot{}K\log^{2}(TK/\delta_0)\log{}d_{\mh_k}}},
\]
since $\cIbar_k$ is precisely the set of rounds for which the \expix{}
instance was active in epoch $k$. Taking a union bound
over all $M$ \expix{} instances and all possible starting times for
each instances, we have that with probability at least
$1-MT\delta_0\geq{}1-\delta/10$, the inequality above holds for all
$k$. 

For $A_5$, let $i,j\in\brk*{M}$ and the interval $\cI=\crl*{t_1,\ldots,t_2}\subset\brk*{T}$
be fixed. The, since $\ls_t(a)\sim\subG(4\tau^{2})$ for all
$a$. Hoeffding's inequality implies that with probability at least
$1-\delta_{0}$,
\[
\sum_{t\in\cI}\ls_t(\pisq_i(x_t))-\ls_t(\pisq_j(x_t)) \leq{}
\abs*{\cI}\cdot{}(L(\pisq_i) - L(\pisq_j)) +O\prn*{
\tau\sqrt{\abs*{\cI}\log(2/\delta_0)}
}.
\]
By a union bound over all $i,j$ pairs and all such intervals, we have
that $A_5$ occurs with probability at least $1-M^{2}T^{2}\delta_0\geq{}1-\delta/10$.
Taking a union bound over events $A_1$ through $A_5$ leads to the final result.
\end{proof}

\subsubsection{Final bound}
We now use the regret decomposition \pref{eq:regret_decomposition}
in conjunction with
\pref{lem:good_event} to prove the theorem. We use $\Otilde$ to
suppress factors logarithmic in $K$, $T$, $M$, and $\log(1/\delta)$.

Condition on the event $A$ in \pref{lem:good_event}, which happens with probability at least
$1-\delta$ so long as $C_1$ and $C_2$ are sufficiently large absolute
constants.

We begin from the regret decomposition
\[
\Reg\leq{} \sum_{k=1}^{k_0-1}\Reg(\cI_k) +
\sum_{k=k_0}^{k_1-1}\Reg(\cI_k) +
\sum_{k=k_1}^{N}\Reg(\cI_k).
\]
We first handle regret from intervals before $k_0$, which is the
simplest case. Observe that that $\sum_{i=1}^{k_0-1}\abs*{\cI_k}\leq{}\Tmin_{\ms}$. Combined with
event $A_1$, this implies that 
\begin{equation}
\sum_{k=1}^{k_0-1}\Reg(\cI_k) \leq{} 2T_{k_0}\cdot\max_{a\in\cA,t\in\brk*{T}}\abs*{\ls_t(a)} \leq{} \Otilde\prn*{
\frac{\tau^{3}}{\gamma^{2}}\cdot{}d_{\ms}\log^{3/2}(2/\delta)
+\tau\log^{2}(2/\delta) + \tau{}K\log^{1/2}(2/\delta)},\label{eq:main1}
\end{equation}
whenever $\expexp\leq{}1/3$. For every other interval, we bound the regret as follows:
\begin{align*}
\Reg(\cI_k) &= \sum_{t\in\cI_k}\ls_{t}(a_t) - \ls_{t}(\pis(x_t))\\
&= \sum_{t\in\cI_k}\ls_{t}(a_t) - \ls_{t}(\pisq_{\ms}(x_t))
  \quad\text{(using \pref{prop:beta_basics})}\\
&= \sum_{t\in\cI_k}\ls_{t}(a_t) - \ls_{t}(\pisq_{\mh_k}(x_t))  + \sum_{t\in\cI_k}\ls_{t}(\pisq_{\mh_k}(x_t)) - \ls_{t}(\pisq_{\ms}(x_t)).
\end{align*}
For the first summation, using event $A_1$ and $A_4$, we have
\begin{align*}
\sum_{t\in\cI_k}\ls_{t}(a_t) - \ls_{t}(\pisq_{\mh_k}(x_t))
&\leq{} \sum_{t\in\cIbar_k}\ls_{t}(a_t) - \ls_{t}(\pisq_{\mh_k}(x_t))
  +
  \Otilde\prn*{\abs*{\cI_k\cap{}S_{T}}\cdot{}\tau\sqrt{\log(2/\delta)}}\\
&\leq{} \Otilde\prn*{\tau\sqrt{d_{\mh_k}\abs*{\cI_k}\cdot{}K\log^{2}(2/\delta)}} + \Otilde\prn*{\abs*{\cI_k\cap{}S_{T}}\cdot{}\tau\sqrt{\log(2/\delta)}}.
\end{align*}
The second summation is exactly zero when $k\geq{}k_1$, otherwise we invoke event $A_5$ and
\pref{lem:square_loss_translation}, which imply
\begin{align*}
  \sum_{t\in\cI_k}\ls_{t}(\pisq_{\mh_k}(x_t)) -
  \ls_{t}(\pisq_{\ms}(x_t))
&\leq{} \abs*{\cI_k}\cdot{}(L(\pisq_{\mh_k}) - L(\pisq_{\ms_k})) +\Otilde\prn*{
\tau\sqrt{\abs*{\cI_k}\log(2/\delta)}
} \\
&\leq{} O\prn*{\abs*{\cI_k}\sqrt{K\cdot\cE_{\mh_k,\ms}}} +\Otilde\prn*{
\tau\sqrt{\abs*{\cI_k}\log(2/\delta)}
}.
\end{align*}
Combining these results, we get that
\begin{align}
&\sum_{k=k_0}^{k_1-1}\Reg(\cI_k) +
\sum_{k=k_1}^{N}\Reg(\cI_k)\notag\\
&\leq{}
\Otilde\prn*{\abs*{S_{T}}\cdot{}\tau\sqrt{\log(2/\delta)}}
+ \Otilde\prn*{
\sum_{k=k_1}^{N}\tau\sqrt{d_{\mh_k}\abs*{\cI_k}\cdot{}K\log^{2}(2/\delta)}
}\notag
\\
&~~~~+
\Otilde\prn*{\sum_{k=k_0}^{k_1-1}\tau\sqrt{d_{\mh_k}\abs*{\cI_k}\cdot{}K\log^{2}(2/\delta)}
+ \abs*{\cI_k}\sqrt{K\cdot\cE_{\mh_k,\ms}}
+\tau\sqrt{\abs*{\cI_k}\log(2/\delta)}}.\label{eq:main2}
\end{align}
From here we split into two cases.
\paragraph{Regret after $k_1$}
We first claim that if $\mh_{k}\geq{}\ms$, it must be the case that
$k=N$, or in other words, if it happens that we switch to a policy class containing
$\pis$, we never leave this class. Indeed, suppose that
$\mh_{k}\geq{}\ms$, and that at time $T_{k+1}$ we switched to $\mh_{k+1}\neq{}\mh_k$. Then
it must have been the case that for $t=T_{k+1}-1$, there was some $i$
for which
\[
\wh{\cE}_{\mh,i}(t) \geq{} 2\alpha_{i,t}.
\]
But since we have $t\geq{}\Tmin_{i}$, event $A_3$ then implies that
\[
\frac{3}{2}\cE_{\mh,i} \geq{} \wh{\cE}_{\mh,i}(t) - \alpha_{i,t}
\geq{}\alpha_{i,t} > 0,
\]
which is a contradiction because $\cE_{i,j}=0$ for all $\ms\leq{}i<j$
by \pref{prop:beta_basics}. We conclude that
\[
\sum_{k=k_1}^{N}\tau\sqrt{d_{\mh_k}\abs*{\cI_k}\cdot{}K\log^2(2/\delta)}
= \tau\sqrt{d_{\mh_{N}}\abs*{\cI_{N}}\cdot{}K\log^2(2/\delta)},
\]
in the case where $k_1=N$, and is zero otherwise. It remains to show
that if we happen to overshoot the class $\ms$ (i.e., $\mh_{N}>\ms$),
then $d_{\mh_{N}}$ is not too large relative to $d_{\ms}$. 

Suppose $\mh_{N}>\ms$, let $j\ldef\mh_{k_1}$ and consider
the epoch prior
to $N$. At the time $t=T_{N}-1$ at which we switched, the definition of $\mh$ implies
that we must have had
\[
\wh{\cE}_{\mh,\ms}(t) \leq{} 2\alpha_{\ms,t}\quad\text{and}\quad\wh{\cE}_{\mh,j}(t)\geq{}2\alpha_{j,t}.
\]
Using event $A_3$, this implies that
$\cE_{\mh,\ms}\leq{}6\alpha_{\ms,t}$ and
$\frac{3}{2}\cE_{\mh,j}\geq{}\alpha_{j,t}$. However,
\pref{prop:beta_basics} implies $\cE_{\mh,j}=\cE_{\mh,\ms}$,
and so we get that
$\alpha_{j,t}\leq{}9\alpha_{\ms,t}$. Expanding out the
definition for the $\alpha_{m,t}$, and defining
$c_1=\frac{\tau^{6}}{\gamma^{4}}\cdot\frac{1}{K^{\expexp}t^{1-\expexp}}$
and $c_2= \frac{\tau^{10}}{\gamma^{8}}\cdot{}\frac{\log(2/\delta_0)}{t}$,
this implies
\[
C_1\cdot\prn*{c_1\cdot{}d_j^{1/2}\log^{2}(2d_j/\delta_0)+c_2\cdot{}d_j}
\leq{} 9C_1\cdot\prn*{c_1\cdot{}d_{\ms}^{1/2}\log^{2}(2d_{\ms}/\delta_0)+c_2\cdot{}d_{\ms}}
\]
or, simplifying,
\[
c_1\cdot{}d_{j}^{1/2}\log^{2}(2d_{j}/\delta_0) \leq{} 9c_1\cdot{}d_{\ms}^{1/2}\log^{2}(2d_{\ms}/\delta_0)+c_2\prn*{9d_{\ms}-d_{j}}
\]
Now consider two cases. If $d_j<9d_{\ms}$ we are done. If not, the
inequality above implies $d_{j}^{1/2}\log^{2}(2d_{j}/\delta_0) \leq{}
d_{\ms}^{1/2}\log^{2}(2d_{\ms}/\delta_0)$. We conclude that
$d_j=O(d_{\ms})$, so 
\begin{equation}
\sum_{k=k_1}^{N}\tau\sqrt{d_{\mh_k}\abs*{\cI_k}\cdot{}K\log^2(2/\delta)}
=
O\prn*{\tau\sqrt{d_{\ms}\abs*{\cI_{N}}\cdot{}K\log^2(2/\delta)}}.\label{eq:main3}
\end{equation}

\paragraph{Regret between $k_0$ and $k_1$}
Let $k_0\leq{}k<k_1$. We will bound the term
\[
\tau\sqrt{d_{\mh_k}\abs*{\cI_k}\cdot{}K\log^2(2/\delta)}
+ \abs*{\cI_k}\sqrt{K\cdot\cE_{\mh_k,\ms}}
\]
appearing in \pref{eq:main2}. For the first term, note that we
trivially have $d_{\mh_k}\leq{}d_{\ms}$. For the second,
consider $t=T_{k+1}-2$. Since we did not switch at this time, we have
$\wh{\cE}_{\mh_k,\ms}(t)\leq{}2\alpha_{\ms,t}$. Combined with
event $A_3$, since $t\geq{}T_{\ms}^{\min}-1$, this implies that
$\cE_{\mh_{k},\ms}\leq{}6\alpha_{\ms,t}$, and so
\begin{align}
\abs*{\cI_k}\sqrt{K\cdot\cE_{\mh_k,\ms}}
&\leq{} 
\Otilde\prn*{\frac{\tau^{3}}{\gamma^{2}}\cdot \abs*{\cI_k}\frac{K^{\frac{1}{2}(1-\expexp)}d_{\ms}^{1/4}\log(2/\delta)}{T_{k+1}^{\frac{1}{2}(1-\expexp)}}
  }
+ 
\Otilde\prn*{\frac{\tau^{5}}{\gamma^{4}}\cdot{}\abs*{\cI_k}\sqrt{\frac{Kd_{\ms}\log(2/\delta)}{T_{k+1}}}}\notag
  \\
&\leq{} 
\Otilde\prn*{\frac{\tau^{3}}{\gamma^{2}}\cdot K^{\frac{1}{2}(1-\expexp)}\abs*{\cI_k}^{\frac{1}{2}(1+\expexp)}d_{\ms}^{1/4}\log(2/\delta)
  }
+ 
\Otilde\prn*{\frac{\tau^{5}}{\gamma^{4}}\cdot{}\sqrt{\abs*{\cI_k}Kd_{\ms}\log(2/\delta)}}.\label{eq:main4}
\end{align}

\paragraph{Final result}
We combine equations \pref{eq:main1}, \pref{eq:main2},
\pref{eq:main3}, and \pref{eq:main4}, and use the bound on
$\abs*{S_T}$ from event $A_2$ to get
\begin{align*}
\Reg
&\leq {}
\Otilde\prn*{\sum_{k=1}^{N}\tau\sqrt{K \abs*{\cI_{k}} d_{\ms}\log^{2}(2/\delta)}
+\tau\sqrt{\log(2/\delta)}K^{\expexp}T^{1-\expexp}
+ \frac{\tau^{3}}{\gamma^{2}}\cdot{}d_{\ms}\log^{3/2}(2/\delta)}\\
&+~~~~ \Otilde\prn*{\sum_{k=1}^{N}\frac{\tau^{3}}{\gamma^{2}}\cdot
 K^{\frac{1}{2}(1-\expexp)}\abs*{\cI_k}^{\frac{1}{2}(1+\expexp)}d_{\ms}^{1/4}\log(2/\delta)
+ \sum_{k=1}^{N}\frac{\tau^{5}}{\gamma^{4}}\cdot{}\sqrt{K\abs*{\cI_k}d_{\ms}\log(2/\delta)}}. 
\end{align*}
Using that $\tau/\gamma\geq{}1$ and Jensen's inequality, this
simplifies to an upper bound of
\begin{align*}
&\leq {}
\Otilde\prn*{
\tau\sqrt{\log(2/\delta)}K^{\expexp}T^{1-\expexp}
+\frac{\tau^{3}}{\gamma^{2}}\cdot
 K^{\frac{1}{2}(1-\expexp)}(T\ms)^{\frac{1}{2}(1+\expexp)}d_{\ms}^{1/4}\log(2/\delta)
+ \frac{\tau^{5}}{\gamma^{4}}\cdot{}\sqrt{KT\ms{}d_{\ms}\log^{2}(2/\delta)}}. 
\end{align*}
For the choice $\kappa=1/3$, this becomes
\begin{align*}
 \Reg &\leq{} \Otilde\prn*{
\tau\sqrt{\log(2/\delta)}K^{1/3}T^{2/3}
+\frac{\tau^{3}}{\gamma^{2}}\cdot
  K^{1/3}(T\ms)^{2/3}d_{\ms}^{1/4}\log(2/\delta)
+
        \frac{\tau^{5}}{\gamma^{4}}\cdot{}\sqrt{KT\ms{}d_{\ms}\log^2(2/\delta)}}
  \\
&\leq{} \Otilde\prn*{
  \frac{\tau^{3}}{\gamma^{2}}\cdot
  (Kd_{\ms})^{1/3}(T\ms)^{2/3}\log(2/\delta)
+ \frac{\tau^{5}}{\gamma^{4}}\cdot{}\sqrt{KT\ms{}d_{\ms}\log^2(2/\delta)}}. 
\end{align*}
Whenever this regret bound is non-trivial, both terms above can be
upper bounded as
\[
 \Reg \leq{} \Otilde\prn*{
  \frac{\tau^{4}}{\gamma^{3}}\cdot
  (Kd_{\ms})^{1/3}(T\ms)^{2/3}\log(2/\delta)}.
\]
For the choice $\kappa=1/4$, we have
\begin{align*}
\Reg
&\leq {}
\Otilde\prn*{
\tau\sqrt{\log(2/\delta)}K^{1/4}T^{3/4}
+\frac{\tau^{3}}{\gamma^{2}}\cdot
  K^{3/8}(T\ms)^{5/8}d_{\ms}^{1/4}\log(2/\delta)
+ \frac{\tau^{5}}{\gamma^{4}}\cdot{}\sqrt{KT\ms{}d_{\ms}\log^2(2/\delta)}}. 
\end{align*}
To simplify the middle term above we consider two cases. If
$K^{1/8}d_{\ms}^{1/4}\leq{}(T\ms)^{1/8}$, then
$K^{3/8}(T\ms)^{5/8}d_{\ms}^{1/4}\leq{}K^{1/4}(T\ms)^{3/4}$. If this
does not hold, then we have
\[
K^{3/8}(T\ms)^{5/8}d_{\ms}^{1/4}
= K^{3/8}\sqrt{T\ms}d_{\ms}^{1/4}\cdot(T\ms)^{1/8}
\leq{} \sqrt{KT\ms{}d_{\ms}}.
\]
Combining these results and using again that $\tau/\gamma\geq{}1$, we have
\begin{align*}
\Reg
&\leq {}
\Otilde\prn*{
\frac{\tau^{3}}{\gamma^{2}}\cdot{}K^{1/4}(T\ms)^{3/4}\log(2/\delta)
+ \frac{\tau^{5}}{\gamma^{4}}\cdot{}\sqrt{KT\ms{}d_{\ms}}\log(2/\delta)}.\tag*\qed
\end{align*}

\subsection{Proofs for remaining results}

\begin{proof}[\pfref{thm:model_selection_log}]
This result is a fairly immediate consequence of
\pref{thm:model_selection_main}. Let $i\geq{}1$ be such that
$e^{i-1}\leq{}d_{\ms}\leq{}e^{i}$. Then the feature map
$\bar{\phi}^{i}$ must not have been removed by the duplicate removal
step. Moreover, since $\bar{\phi}^{i}$ was chosen to be the largest
feature map with dimension bounded by $e^{i}$, the policy class it
induces must contain the policy class induced by $\phi^{\ms}$ by nestedness, and the realizability assumption is preserved
by the new set of feature maps. Lastly, we observe that $i\leq{}\log(d_{\ms})\leq{}\log(T)$, and
that $d_{i}\leq{}e\cdot{}d_{\ms}$, leading to the
result.
\end{proof}





\end{document}